\documentclass{article}

\usepackage[final]{neurips_2024}

\usepackage[utf8]{inputenc} \usepackage[T1]{fontenc}    \usepackage[hidelinks,colorlinks=true,linkcolor=blue,citecolor=blue]{hyperref}
\usepackage{url}            \usepackage{booktabs}       \usepackage{amsfonts}       \usepackage{nicefrac}       \usepackage{microtype}      \usepackage[dvipsnames]{xcolor}         \usepackage{wrapfig}

\usepackage{tikz}
\usetikzlibrary{automata}
\usetikzlibrary{positioning}  \usetikzlibrary{arrows}       \tikzset{node distance=1.5cm, every state/.style={ semithick,
           fill=gray!5},
         initial text={},     double distance=3pt, every edge/.style={  draw,
           ->,>=stealth',     auto,
           semithick}}
\usepackage{xspace}
\usepackage{caption}
\usepackage{subcaption}

\usepackage{algorithm}
\usepackage{algorithmic}

\usepackage{amsmath}
\usepackage{amssymb}
\usepackage{url}
\usepackage{amsthm}

\usepackage{bm}
\usepackage[dvipsnames]{xcolor} \usepackage{cleveref}
\newtheorem{theorem}{Theorem}[section]
\newtheorem{proposition}[theorem]{Proposition}
\newtheorem{lemma}[theorem]{Lemma}

\newtheorem{definition}[theorem]{Definition}
\newtheorem{assumption}[theorem]{Assumption}

\newtheorem{remark}[theorem]{Remark}

\crefname{assumption}{assumption}{assumptions}
\Crefname{assumption}{Assumption}{Assumptions}
\crefname{conjecture}{conjecture}{conjectures}
\Crefname{conjecture}{Conjecture}{Conjectures}

\newcommand{\yhat}{\widehat{y}}

 \newcommand{\E}{\mathbb{E}}

\providecommand{\argmin}{\mathop{\mathrm{argmin}}}

\providecommand{\E}{\mathbb{E}}

\providecommand{\yhat}{\hat{\yvec}}

\definecolor{dg}{RGB}{0,120,0}

\newcommand{\cf}{Counterfactual Fairness}

\newcommand{\cls}{\phi}

\newcommand{\pcf}{\textnormal{PCF}}
\newcommand{\PCF}{\textnormal{PCF}}
\newcommand{\CF}{\textnormal{CF}}

\newcommand{\train}{\textnormal{train}}
\newcommand{\test}{\textnormal{test}}
\newcommand{\te}{\textnormal{TE}}

 \newcommand{\risk}{\mathcal{R}}
 
\newcommand{\U}{\bf{U}}
\newcommand{\V}{\bf{V}}
\newcommand{\TE}{\textnormal{TE}}

\definecolor{amaranth}{rgb}{0.9, 0.17, 0.31}

\renewcommand{\E}{\mathbb{E}}

\title{Counterfactual Fairness by Combining Factual and Counterfactual Predictions}

\author{Zeyu Zhou, Tianci Liu, Ruqi Bai, Jing Gao, Murat Kocaoglu, David I. Inouye  \\
  Elmore Family School of Electrical and Computer Engineering\\
  Purdue University\\
  \texttt{\{zhou1059,  liu3351, bai116, jinggao, mkocaoglu, dinouye\}@purdue.edu} \\
}

\begin{document}

\maketitle

\begin{abstract}
In high-stakes domains such as healthcare and hiring, the role of machine learning (ML) in decision-making raises significant fairness concerns. 
This work focuses on Counterfactual Fairness (CF), which posits that an ML model's outcome on any individual should remain unchanged if they had belonged to a different demographic group.
Previous works have proposed methods that guarantee CF. 
Notwithstanding, their effects on the model's predictive performance remain largely unclear.
To fill this gap, 
we provide a theoretical study on the inherent trade-off between CF and predictive performance in a model-agnostic manner. 
We first propose a simple but effective method to cast an optimal but potentially unfair predictor into a fair one with minimal performance degradation.
By analyzing the excess risk incurred by perfect CF, we quantify this inherent trade-off. 
Further analysis on our method's performance with access to only incomplete causal knowledge is also conducted. 
Built upon this, we propose a practical algorithm that can be applied in such scenarios. 
Experiments on both synthetic and semi-synthetic datasets demonstrate the validity of our analysis and methods.

\end{abstract}

\section{Introduction}\label{sec:intro}

Machine learning (ML) has been widely used in high-stakes domains such as healthcare \citep{daneshjou2021disparities}, hiring \citep{hoffman2018discretion}, criminal justice \citep{brennan2009evaluating}, and loan assessment \citep{khandani2010consumer}, bringing with it critical ethical and social considerations. 
A prominent example is the bias observed in the COMPAS tool against African Americans in recidivism predictions \citep{brackey2019analysis}. 
This issue is particularly alarming in an era where large-scale deep learning models, commonly trained on noisy data from the internet, are increasingly prevalent. 
Such models, due to their extensive reach and impact, amplify the potential for widespread and systemic biases.
This increasing awareness underscores the need for ML practitioners to integrate fairness considerations into their work, extending their focus beyond merely maximizing prediction accuracy \citep{bolukbasi2016man, calders2010three, dwork2012fairness, grgic2016case, hardt2016equality}. 
Various fairness notions have been developed, ranging from group-level measures such as group parity \citep{hardt2016equality} to individual-level metrics \citep{dwork2012fairness}. 
Recently, there has been a growing interest in approaches based on causal inference, particularly in understanding the causal effects of sensitive attributes such as \textit{gender} and \textit{age} on decision-making \citep{chiappa2019path, galhotra2022causal, khademi2019fairness}. 
This has led to the proposal of Counterfactual Fairness (CF), which states that prediction for an individual in hypothetical scenarios where their sensitive attributes differ should remain unchanged \citep{kusner2017counterfactual}.
As an individual-level notion agnostic to the choice of similarity measure \citep{kusner2017counterfactual, rosenblatt23counterfactual}, CF has recently gained traction \citep{anthis2024causal, nilforoshan2022causal, makhlouf2022survey, rosenblatt23counterfactual}.

\begin{figure}[!ht]
    \centering
    \begin{subfigure}[t]{0.4\linewidth}
        \centering
        \includegraphics[width=0.75\linewidth]{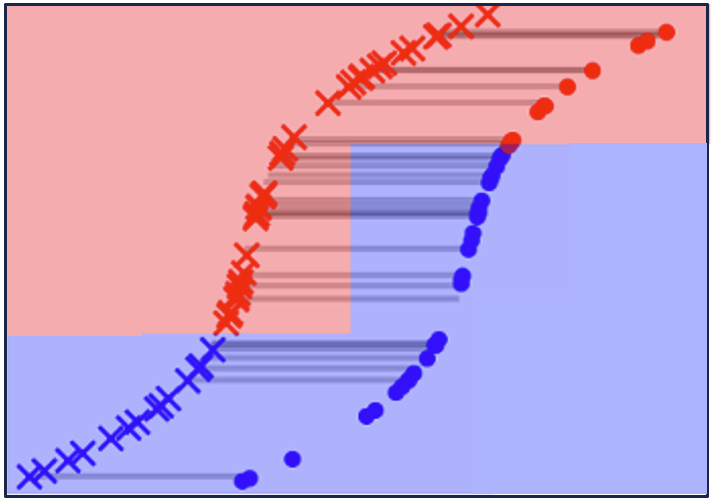}
        \caption{Optimal predictor without counterfactual fairness constraint.}
        \label{fig:estana_linear_knn_cf_effect}
    \end{subfigure}
    \hspace{2em}
    \begin{subfigure}[t]{0.45\linewidth}
        \centering
\includegraphics[width=\linewidth]{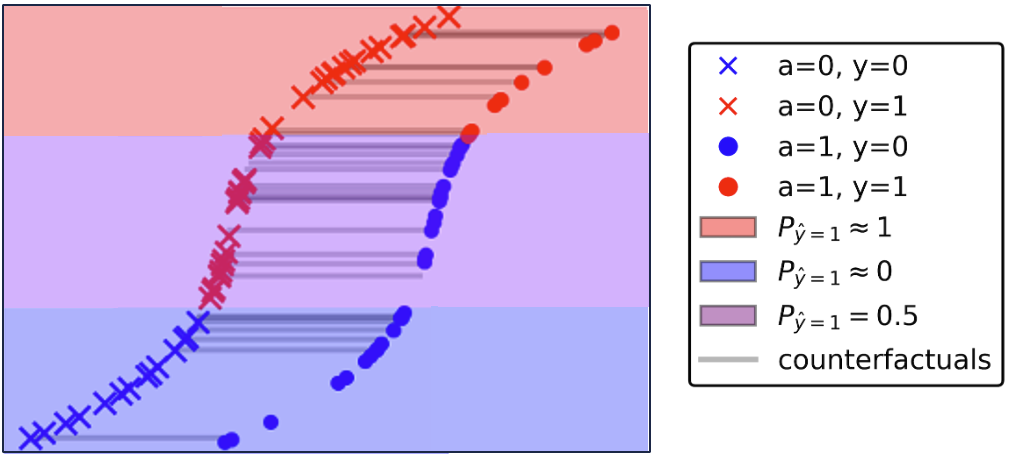}
        \caption{Optimal predictor under perfect counterfactual fairness constraint.}        \label{fig:estana_cubic_knn_cf_effect}
    \end{subfigure}
    
    \caption{
    The optimal (unfair) predictor (a) violates counterfactual fairness in the middle region because the predictions are different for the factual-counterfactual pairs \protect\footnotemark (denoted by line segments between $a=0$ and $a=1$). We prove that the optimal fair predictor (b) simply mixes the optimal unfair predictions at the factual and counterfactual points (i.e., mixes the predictions at both endpoints of the line). 
    This mixing incurs the inherent excess risk associated with counterfactual fairness.
    Colors represent target classes ($Y$), and dot styles represent sensitive attributes ($A$).
    }
    \label{fig:cf-ill}
    \vspace{-1.5em}
\end{figure}

To achieve CF, \citet{kusner2017counterfactual} first proposed a naive solution, suggesting that predictions should only use non-descendants of the sensitive attribute in a causal graph. 
This approach only requires a causal topological ordering of variables and achieves perfect CF by construction.
However, it limits the available features for downstream tasks and could be inapplicable in certain cases \citep{kusner2017counterfactual}. 
For relaxation, they further proposed an algorithm that leverages latent variables.
Extending this line of work, \citet{zuo2023counterfactually} introduced a technique that incorporates additional information by mixing factual and counterfactual samples. 
Although perfect CF has been established in their work, 
its predictiveness degraded, and whether the predictive power can be improved remains unknown.
In parallel to this, another branch of research employed regularization and augmentation to \textit{encourage} CF \citep{garg2019counterfactual,stefano2020removing,kim2021counterfactual}. 
However, as these methods cannot guarantee perfect CF, analyzing the optimal predictive performance under CF constraints is highly challenging. 
\footnotetext{Because of our invertibility assumption, any factual point has a unique counterfactual.}

To theoretically understand the tradeoff between CF and ML performance, we consider a class of invertible causal models and prove that the optimal solution under perfect Counterfactual Fairness (CF) has a simple form w.r.t. to the Bayes optimal classifier and explicitly quantify the excess risk of imposing a perfect CF constraint as has been done for non-causal fairness notions\citep{zhao2022inherent,xian2023fair}.
The optimal predictor under the fairness constraint can be achieved by combining factual and counterfactual predictions using a (potentially unfair) optimal predictor. 
Next, we quantify the excess risk between the optimal predictor with and without CF constraints.
This quantity sheds light on the best possible model, in terms of predictive performance, under the stringent notion of perfect CF. 
Our results are illustrated in \Cref{fig:cf-ill}.
To consider scenarios with incomplete causal knowledge (e.g. unknown causal graph or model),
we further study the CF and predictive performance degradation caused by imperfect counterfactual estimations.
Inspired by our theoretical findings, we propose a plugin method that leverages a (potentially unfair) pretrained model to achieve a better tradeoff of fairness and predictive performance than the prior methods.
Furthermore, we propose a method to improve the pretrained model that accounts for counterfactual estimation errors and can achieve good empirical performance even with limited causal knowledge.
We summarize our contributions as follows:

\begin{enumerate}
\item We propose a CF method that is provably optimal in terms of predictive performance under perfect CF. 
    \item 
    To the best of our knowledge, we are the first to characterize the inherent trade-off of CF and ML performance, which applies to all CF methods. 
     \item We investigate the CF and predictive performance degradation from estimation error resulting from limited causal knowledge and propose methods to mitigate estimation errors in practice.
     \item We empirically demonstrate that our proposed CF methods outperform existing methods in both full and incomplete causal knowledge settings \footnote{Code can be found in \href{https://github.com/inouye-lab/pcf}{https://github.com/inouye-lab/pcf}}.
\end{enumerate}

\section{Preliminaries}\label{sec:prelim}

\paragraph{Notation} We use capital letters to represent random variables and lowercase letters to represent the realizations of random variables.
Now we define a few variables that will be considered in this work.
$A$ represents the sensitive attribute of an individual (e.g., gender), $Y$ represents the target variable to predict, $X$ represents observed features other than $A$ and $Y$, and $U$ represents unobserved confounding variables which are not caused by any observed variables while $a,y,x,u$ represent their realization respectively.

\paragraph{Counterfactual} In this work, we use the framework of Structural Causal Models (SCMs) \citep{pearl2009causality}.
A SCM is a triplet $\mathcal{M}=(\U,\V,\mathcal{F})$ where $\U$ represents exogenous variables (factors outside the model), $\V$ represents endogenous variables, and $\mathcal{F}$ contains a set of functions $F_i$ that map  from $U_i$ and $Parent(V_i)$ to $V_i$.
A counterfactual query asks a question like: what would the value of $Y$ be if $A$ had taken a different value given certain observations?
For example, given that a person is a woman and given everything we observe about her performance in an interview, what is the probability of her getting the job if she had been a man?
More formally, given a SCM, a counterfactual query can be written as $P(Y_{A=a}|W=w)$.
Here $W=w$ is the evidence and $A=a$ in the subscript represents the intervention on $A$.
For the general procedure to estimate counterfactuals, please refer to \citet{pearl2009causality}. 

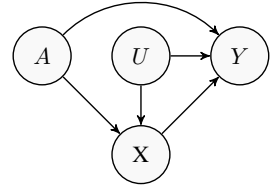
\begin{wrapfigure}{r}{0.25\textwidth}
    \centering
    \resizebox{\linewidth}{!}{
    \begin{tikzpicture}
    \node[state] (Zd) {$A$};
    \node[state, right of=Zd] (Zy) {$U$};
    \node[state, right of=Zy] (Zres) {$Y$};
    \node[state, below of=Zy] (X) {X};

\path (Zy) edge [] node [above] {} (Zres) ;
    \path (Zd) edge [] node [above] {} (X) ;
    \path (Zy) edge [] node [above] {} (X) ;
    \path (X) edge [] node [above] {} (Zres) ;
    \path (Zd) edge [bend left=45] node [above] {} (Zres) ;
    \end{tikzpicture}}
\caption{Causal graph. $A$ represents sensitive attribute, $Y$ represents the target variable, $U$ represents latent confounders, $X$ represents observed features. 
Note that the validity of our theoretical analysis holds for all causal models that satisfy the condition given by \Cref{asm:main}. It is not restricted to this specific graph.}
    \label{fig:causal-graph}
\end{wrapfigure}

\paragraph{Counterfactual Fairness}
Built upon the framework above, we focus on Counterfactual Fairness (CF), which requires the predictors to be fair among factual and counterfactual samples. 
More formally, it is defined as below
\begin{definition}(Counterfactual Fairness)
    We say a predictor $\hat Y$ is counterfactually fair if $$p(\hat{Y}_{A=a}|X=x,A=a)= p(\hat{Y}_{A=a'}|X=x,A=a), \quad \forall (x, a).$$
\end{definition}
This definition states that intervention on $A$ should not affect the distribution of $\hat{Y}$.
Using the same example above, the probability of a woman getting the job should be the same as that if she had been a man. 
For that goal, we use the following metric to evaluate CF
\begin{definition}(Total Effect)
The Total Effect (TE) of a predictor $\hat{Y}$ is 
    $$ \TE \triangleq \E[|\hat{Y}_{A=a} - \hat{Y}_{A=a'}|].
    $$
\end{definition}
Therefore, a predictor is counterfactually fair if and only if $\TE=0$.
Throughout the paper, we use TE to quantify the violation of counterfactual fairness.

\section{Counterfactual Fairness via Output Combination}\label{sec:fair-alg}
\label{sec:method}
\subsection{Problem Setup}
We assume that all data we have is generated by a causal model \citep{pearl2009causality}, and we consider the representative causal graph shown in \Cref{fig:causal-graph} that has been widely adopted in the fairness literature \citep{DBLP:journals/ml/GrariLD23,kusner2017counterfactual,zuo2023counterfactually}.
Our analysis is presented based on binary $A \in \{0, 1\}$ given its pivotal importance in the literature \citep{pessach2023algorithmic} and for the sake of presentation clearness, but our analysis and our method can be naturally extended to multi-class $A$.
We first state the main assumptions needed in this section.

\begin{assumption}\label{asm:main}\text{}
    \begin{enumerate}
        \item $A$ and $U$ are independent of each other.
        \item The mapping between $X$ and $U$ is invertible given $A$.
    \end{enumerate}
\end{assumption}

The first assumption is very common in the fairness literature.
While the invertibility assumption might be restrictive in certain scenarios, it simplifies the theoretical analysis and has been adopted in recent works on counterfactual estimation \citep{nasr2023counterfactual, kulinski2023towards}.
We expect that exact invertibility is not required in practice but rather only strong mutual information between $X$ and $U$ given $A$ would be sufficient.
Further, we empirically validate the effectiveness of our method after relaxing invertibility in the experiment section.
To facilitate our discussion, we first define $F_X$ as the mapping between $X$ and $(U, A)$, i.e., $X = F_X(U, A)$. 
According to our second assumption, $F_X(\cdot, a)$ is an invertible function, i.e., $\exists {F_X^*}^{-1}, {F_X^*}^{-1}(x,a)={F_X^*}^{-1}(F_X^*(u,a),a) = u, \forall (x,a)$.
This assumption simplifies the counterfactual estimation of $X$ for different values of $A$ into a deterministic function. 
In our context, the counterfactual query is specifically $p(X_{a'} | X=x, A=a)$, which simplifies to a Dirac delta at a single value given the invertibility assumption.
Thus, we introduce the concept of a deterministic counterfactual generating mechanism (CGM), denoted as $x_{a'} = G(x, a, a')$.
Also, in our case, we will assume $A$ is binary so that $a$ and it's counterfactual $a'$ can be written as $1-a$.
All proofs can be found in \Cref{app-sec:proof}.
Given this setup, the following lemma characterizes the perfect CF constraint on $\phi$.

\begin{lemma}
\label{thm:perfect-te}
    Given \Cref{asm:main}, predictor $\phi$ on $(X,A)$ is counterfactually fair if and only if the predictor returns the same value for a sample and its counterfactuals, i.e., $\mathrm{TE}(\phi)=0 
        \Leftrightarrow \phi(x,a) \overset{\text{a.s.}}{=} \phi(x_{1-a}, 1-a), \quad\forall (x,a)$.
\end{lemma}
The proof is straightforward from the definition of TE. 
Notably, this lemma helps disambiguate the question of whether counterfactual fairness is a distribution- or individual-level requirements as raised in \citet{plecko2022causal}. 
In our setup, they are equivalent due to invertibility between $X$ and $U$ given $A$.

\subsection{Optimal Counterfactual Fairness and Inherent Trade-off}\label{sec:method-pretrained}

Given the complete knowledge of the causal model, it is viable to satisfy the perfect CF constraint \citep{kusner2017counterfactual, zuo2023counterfactually}. 
However, these methods are known to result in \textit{empirical} degradation of ML models' performance, raising critical concerns about the \textit{fairness-utility trade-offs}. 
Moreover, it is still unknown \textit{to what extent the ML model performance has to be affected in order to achieve perfect CF}. 
In this section we provide a formal study on this to close the gap. 
Our solution consists of two steps. 
First, we propose a simple yet effective method that is provably optimal under the constraint of perfect CF. 
Next, we characterize the inherent trade-off between CF and predictive performance by checking the excess risk compared to a Bayes optimal (unfair) predictor. 
Our result shows that the inherent trade-off is dominated by the dependency between $Y$ and $A$, 
echoing previous analysis on non-causal based fairness notions \citep{chzhen2020fair,xian2023fair}.
For brevity, we refer to a \textit{Bayes optimal} predictor as ``optimal'', and a model that satisfies \textit{perfect CF} as ``fair''.

We start with the following theorem instantiating an \textit{optimal and fair} predictor.

\begin{theorem}\label{thm:opt-predictor}
    Given \Cref{asm:main} 
    \footnote{Note that while the discussion in this paper focuses on the causal model in \Cref{fig:causal-graph}, our theorem is valid for more general cases that satisfy \Cref{asm:main}. And for brevity, the following discussion is done under this assumption unless otherwise stated.} and loss $\ell$ (i.e., squared $L_2$ loss for regression tasks, and cross-entropy loss for classification tasks), 
    an optimal and fair predictor (i.e., the best possible model(s) under the constraint of perfect CF) is given by the average of the optimal (potentially unfair) predictions on itself and all possible counterfactuals:
    \begin{align*}
        \phi^*_\mathrm{CF}(x,a) 
        &\triangleq p(A\!=\!a)\phi^*(x, a) + p(A\!=\!1\!-\!a) \phi^*(x_{1-a}, 1\!-\!a)
        \in \argmin_{\phi:\mathrm{TE}(\phi)=0}  \E[\ell(\phi(X,A), Y)] \,,
\end{align*}
    where $x_{1-a} = G^*(x,a,1\!-\!a)$ is the counterfactual of $(x,a)$ when intervening with $A=1\!-\!a$, and $\phi^*(x,a)$ is an unconstrained optimal predictor, i.e., $\phi^*(x,a)\triangleq \argmin_{\phi} \E[\ell(\phi(X,A), Y)]\!=\!\E[Y|X\!\!=\!x,A\!=\!a]$.
\end{theorem}

This result suggests that, if we have to access to ground truth counterfactuals, a simple algorithm using a (potentially) unfair model could achieve strong fairness and accuracy.
Built upon the above result, we are ready to characterize the inherent trade-off between CF and model performance by the following theorem.

\begin{theorem}\label{thm:excess-risk-cf}
The inherent trade-off between CF and predictive performance, characterized by the excess risk of the Bayes optimal predictor under the CF constraint, is given by 
$$
\mathcal{R}^*_{\text{CF}} - \mathcal {R}^* 
= \sigma^2_A  \E_{U}\left [\left(\E_{Y\mid U=u, A=a} [Y] - \E_{Y\mid U=u, A=1-a} [Y] \right)^2\right],
$$
for regression tasks using squared $L_2$ loss where $\sigma^2_A$ denotes the variance of $A$; and 
$$
\mathcal{R}^*_{\text{CF}} - \mathcal {R}^* 
= I(A; Y \mid U),
$$
for classification tasks using cross-entropy loss.
\end{theorem} 

Remarkably, the excess risks are completely characterized by the \textit{inherent} dependency between $Y$ and $A$ as determined by the underlying causal mechanism, similar to non-causal based group fairness \citep{chzhen2020fair,xian2023fair}. 
Moreover, 
they lower bound the excess risk of all possible predictors in order to achieve perfect CF.

\subsection{Method with Incomplete Causal Knowledge}\label{sec:practical-alg}

In this section, we aim to address CF in the scenario where causal knowledge is limited.
Inspired by \Cref{thm:opt-predictor}, we first present a simple plugin method as summarized in \Cref{alg:pcf}.
For regression tasks, $\hat{\mu}$ is the final output, and for classification tasks, $\hat{\mu}$ represents the probability of $Y=1$, i.e., $p(Y=1|X=x,A=a)=\E[Y|X=x,A=a]$. 
It is noteworthy that PCF is agnostic to the training of predictor $\phi$ that can be determined by the user freely. 
In fact, with access to the oracle CGM $G^*$, then PCF would achieve perfect CF as proved in the next result.
\begin{algorithm}[!ht]
\caption{Plug-in Counterfactual Fairness (PCF)}
\label{alg:pcf}
\begin{algorithmic}
\STATE {\bfseries Input:} Pretrained probabilistic prediction predictor $\phi: \mathcal{X} \times \mathcal{A} \to \mathcal{Y}$, CGM $G$, test datapoint $(x,a)$, prior distribution $p$ of $A$
\STATE {\bfseries Output:} Predicted output $\hat{\mu}$
\STATE $\hat{x}_{1-a} \gets G(x,a,1-a)$
\STATE $\hat{\mu} \gets p(A=a)\phi(x, a) + p(A=1-a)\phi(\hat{x}_{1-a},1-a)$
\end{algorithmic}
\end{algorithm}
\begin{proposition}\label{thm:pcf-perfect}
    Given that $G$ is the ground truth counterfactual generating mechanism, i.e., $G(x,a,a') = x_{a'}, \forall (x,a,a')$, \Cref{alg:pcf} achieves perfect CF for any pretrained predictor $\phi$.
\end{proposition}
Note that this proposition only requires access to ground truth $G^*$ and holds valid \emph{for any pretrained predictor $\phi$}. 
If $\phi$ is further accurate, then the corresponding PCF is able to achieve high accuracy as well, which is empirically validated in the experiments.

\subsubsection{Given estimated $G$}
Acquiring counterfactuals in practice can be a challenging task and could lead to estimation errors.
In this section we provide a theoretical analysis on this. 
Specifically, the theorem below bounds the TE and excess risk due to the use of estimated counterfactuals.

\begin{theorem}\label{thm:pcf-fair-error-bound}
   Given an optimal predictor $\phi^*(x, a)$, suppose it is L-lipschitz continuous in $x$, and the counterfactual estimation error is bounded, i.e.,
   $$ \max_{X,A}\|G^*(x_a,a,1-a) - \hat{G}(x_a,a,1-a)\|_2 \leq \varepsilon$$ 
   for some $\varepsilon \geq 0$, where $G^*$ and $\hat G$ represent the ground truth and estimated CGMs respectively.
   Then, the total effect (TE) of \Cref{alg:pcf} based on $\hat{G}$ is bounded by $ L\varepsilon$.
   Moreover,
    for squared $L_2$ loss, the excess risk is bounded by $\sigma_A^2  L^2 \varepsilon^2 + 2 \sigma_A^2 L \varepsilon 
     \E_{U} [|\E[Y \mid U=u, A=1] - \E[Y \mid U=u, A=0]|] $, and for cross-entropy loss \footnote{Here we assume the logits are L-lipschitz continuous in $x$.}, the excess risk is bounded by $L\varepsilon$.
\end{theorem}

Note that $\sigma_A$ and $ \E_{U} [|\E[Y \mid U=u, A=1] - \E[Y \mid U=u, A=0]|] $ are inherent characteristic of the underlying mechanism and is independent of the counterfactual estimation.
This suggests that if the counterfactuals are not too far away and $\phi^*$ is smooth, then fairness and prediction performance will not be significantly affected.
In practice, CGM in \Cref{alg:pcf} can be obtained using counterfactual estimation methods, as discussed in \Cref{sec:conclusion}.

\subsubsection{Given estimated $G$ and $\phi$}\label{sec:practical-alg-g-phi}

In the previous section, we discussed how counterfactual estimation error directly impacts the performance of PCF in terms of CF and predictive performance.
Here, we consider the situation where $\phi$ also needs to be estimated. 
We first note that the degradation in fairness remains the same as previous result
\begin{remark}
    The bound of TE given $\hat{\phi}$ and $\hat{G}$ follows that in \Cref{thm:pcf-fair-error-bound}.
\end{remark}

The proof is straightforward since the original proof in \Cref{thm:pcf-fair-error-bound} does not use any characteristic of optimality.
To achieve good predictive performance, a natural approach is to train $\phi$ on the observed data via Empirical Risk Minimization (ERM), which should fit the predictor well given sufficient samples and a reasonable predictor class.
However, ERM can only approximate Bayes optimality within the support of the training data. 
Outside this support, its performance can deteriorate significantly, as extensively studied in areas such as Domain Adaptation \citep{farahani2021brief} and Domain Generalization \citep{zhou2022domain}. 
Consequently, when integrated with an approximate $G$, we may encounter the issue where $p(\hat{X}_{A=1-a}|A=a) \neq p(X|A=1-a)$, inducing a distribution shift problem (Note that these would be equal given the graph in \Cref{fig:causal-graph} and Pearl's rules) \citep{kulinski2023shift}.
To mitigate this, we suggest improving $\phi$ on the estimated counterfactual distribution. 
More formally we define the following objective called Counterfactual Risk Minimization (CRM):
$$
\min_{\phi} \E_{X,A,Y} [\ell(\phi(X,A),Y) + \ell(\phi(G(X,A,1-A),A),Y)] 
$$

This can be achieved either by augmenting the original training dataset or by fine-tuning with estimated counterfactual samples. 
The choice between training from scratch or fine-tuning depends on the scale of the experiment and computational constraints.
 It is important to note that the $Y$ corresponding to the estimated counterfactual should remain the same with that of the factual samples. 
 While the optimal prediction for the counterfactual may differ from that of the original data, under the constraint of perfect CF, a predictor is required to predict the same outcome to counterfactual pairs. 
Hence, the optimal solution will change. This is exactly what causes the excess risk we characterized in \Cref{thm:excess-risk-cf}.
Furthermore, we can prove that, given the ground truth $G$, CRM yields the same optimal solution as PCF.
Since this result is dependent on the ground truth $G$, we provide a more formal statement in \Cref{app-sec:crm} to ensure consistency.

In summary, this section discusses how to improve the estimation of $\phi$ under counterfactual estimation error. 
We propose using data augmentation or fine-tuning based on practical scenarios. 
Additionally, in domains with abundant off-the-shelf pre-trained models \citep{bommasani2021opportunities}, we can potentially avoid this issue by using these models as a good proxy for $\phi$.

\section{Related Works}

\textbf{Fairness Notions} 
Fair Machine Learning has accumulated a vast literature that proposes various notions to measure fairness issues of machine learning models. 
Representative fairness notions can be categorized into three classes. 
\textit{Group fairness}, such as demographic parity \citep{pedreshi2008discrimination} and equalized odds \citep{hardt2016equality}, requires certain group-level statistical independence between model predictions and individuals' demographic information. 
Despite its conceptual simplicity, group fairness is known for ruling out perfect model performance \citep{hardt2016equality} and may allow for bias against certain individuals \citep{corbett2023measure}. 
\textit{Individual fairness} \citep{dwork2012fairness}, on the other hand, asks a model to treat similar individuals similarly. 
However, determining the similarity between different individuals is often highly task-specific and open-ended.
Recently, \textit{counterfactual fairness} (CF, \citet{kusner2017counterfactual}) further takes the causal relationship of data attributes into consideration when measuring fairness. 
In words, counterfactual fairness proposes that a model should treat any individual the same as their \textit{counterfactual} if the individual had been from another demographic group.
As an individual-level notion agnostic to the choice of similarity measure \citep{kusner2017counterfactual},  
CF has recently gained traction \citep{wu2019counterfactual, nilforoshan2022causal, makhlouf2022survey, rosenblatt23counterfactual}.
Motivated by these recent advances, in this work, we focus on the counterfactual fairness. 

\textbf{Methods for Fairness} 
Given an unfair dataset, 
attempts to achieve fairness fall into three categories. 
\textit{Pre-processing} cleans the data before running machine learning models on it, typically by resampling samples or removing undesired attributes \citep{kamiran2012data}.
\textit{In-processing} intervenes the model-training process by incorporating fairness constraints \citep{zafar2017fairness,donini2018empirical,lohaus2020too} or penalties \citep{mohler2018penalized,scutari2021achieving,liu2023simfair}. 
\textit{Post-processing} adjusts the raw model outputs to close the bias gap by, e.g., assigning each demographic group a unique decision threshold \citep{jang2022group}.
Post-processing has been favored as an efficient and practical solution because it does not require retraining the original model \citep{petersen2021post,xian2023fair}.
To achieve CF, \citet{kusner2017counterfactual} applied pre-processing and discarded all descendants of the sensitive feature. 
\citet{chen2024learning} pre-processed the data via orthogonalization and marginal distribution mapping.
\citet{garg2019counterfactual,stefano2020removing,kim2021counterfactual} in-processed the model training by penalizing CF violations but their solutions lack formal CF guarantees and often contain unsatisfactory bias after the intervention \citep{zuo2023counterfactually}.
Recently, \citet{zuo2023counterfactually} proposed another in-processing based solution that is capable of achieving perfect CF and better performance via mixing features.
\citet{ma2023counterfactualfairnesspredictionsusing} leveraged mediators estimated by Generative Adversarial Networks and provided a theoretical guarantee of CF under well-estimated counterfactuals.
However, it is unclear whether their methods are optimal.
\citet{wang2023adjusting} leveraged predictor that satisfies equal counterfactual opportunity criterion to construct a counterfactually fair predictor.
While they provide results on optimality, their findings assume an ideal setting where non-sensitive features are independent of sensitive features.

\textbf{Inherent Trade-off between Fairness and Predictiveness}
Machine learning models are known to suffer from performance drops after fairness interventions \citep{hardt2016equality, menon2018cost, chen2018my}, which is known as the \textit{fairness-utility} trade-offs. 
Recently, inherent trade-offs towards non-causal based fairness such as demographic parity (DP) has been established separately for regression \citep{chzhen2020fair} and classification tasks \citep{xian2023fair}.
The excess risks are characterized by certain distribution distance (i.e., Wasserstein-2 barycenter for regression, and total-variation or Wasserstein-1 barycenter for noiseless or noisy classification) between the conditional distribution of $Y$ given $A$.
A similar trade-off between CF and predictiveness has also been empirically observed \citep{zuo2023counterfactually}. 
Nonetheless, their inherent trade-off remains an open question. 
In this work we take the first step towards this goal and provide a quantitative analysis in both complete and incomplete causal knowledge settings as presented in Section \ref{sec:method}. We hope our work sheds light on future works towards more effective CF.

\section{Experiments}\label{sec:exp}
In this section, we validate our theorems and the effectiveness of our algorithms through experiments on synthetic and semi-synthetic datasets. 
On synthetic datasets, we focus on validating our theorems in settings where our assumptions hold. 
On semi-synthetic datasets, we aim to assess the effectiveness of our methods in more practical scenarios, where limited causal knowledge is available and the invertibility assumption is relaxed.

\paragraph{Metrics}
We consider two metrics in this paper: Error and  Total Effect (TE).
The former evaluates whether each method can achieve its goal, irrespective of fairness.
This is important because we can achieve perfect Counterfactual Fairness by always outputting fixed prediction given whatever input, but that is not useful at all.
The latter is a common metric to evaluate Counterfactual Fairness \citep{kim2021counterfactual,zuo2023counterfactually}.
Given a test set $\mathcal{D}_\test$, Error is defined as $\textnormal{Error} = \frac{1}{|\mathcal{D}_\test|}\sum_{x^{(i)} \in \mathcal{D}_\test} \ell(\yhat (x^{(i)}), y ^{(i)})$ where $y ^{(i)}$ is the ground truth target, $\yhat (x^{(i)})$ is the prediction of $x^{(i)}$, and $\ell$ depends on the task.
TE is defined as $\te = \frac{1}{|\mathcal{D}_\test|} \sum_{x^{(i)} \in \mathcal{D}_\test} |\yhat (x^{(i)}) -\yhat (x_{1-a}^{(i)}) |$ where $x_{1-a}$ is the ground truth counterfactual corresponding to $x^{(i)}$.
Since we only consider binary sensitive attribute, we further define $\te_0  = \frac{1}{|\{i:a^{(i)}=0\}|}\sum_{i:a^{(i)}=0}|\yhat (x^{(i)}) -\yhat (x_{1-a}^{(i)}) | $ and $\te_1  = \frac{1}{|\{i:a^{(i)}=1\}|}\sum_{i:a^{(i)}=1}|\yhat (x^{(i)}) -\yhat (x_{1-a}^{(i)}) | $ to evaluate Counterfactual Fairness for different group respectively.

\paragraph{Methods}
In general, we consider the following methods: 
(1) \textbf{Empirical Risk Minimization (ERM):} Train a classifier on all features without any fairness consideration. 
Specifically $\widehat{y} = \cls(x,a),$ where $\cls$ represents the predictor.
(2) \textbf{Counterfactual Fairness with $U$ (CFU) \citep{kusner2017counterfactual}:} To achieve \cf, CFU proposes to use $U$ for prediction.
Specifically, $\widehat{y} = \cls(u)$.
(3) \textbf{Counterfactual Fairness with fair representation (CFR) \citep{zuo2023counterfactually}:}
CFR proposes to use $U$ and a symmetric version of $x,x_{1-a}$.
Specifically, $\widehat{y} = \cls(\frac{x+x_{1-a}}{2},u)$.
(4) \textbf{Equal Counterfactual Opportunity (ECOCF)\citep{wang2023adjusting}:} 
An ECO predictor is adjusted to become counterfactually fair.
Specifically,
$\hat{y}= p(a)[p(a)\phi(x,a) + (1-p(a))\phi(x,1-a)] + (1-p(a))[(1-p(a))\phi(x_{1-a},1-a) + p(a) \phi(x_{1-a},a)]$
(5) \textbf{PCF\footnote{Essentially PCF with ERM (PCF-ERM). For brevity, we just call it PCF.}:} As introduced in \Cref{alg:pcf}, PCF mixes the output of factual and counterfactual prediction. 
Specifically, $\hat{y} = p(a)\phi(x,a) + (1-p(a))\phi(x_{1-a},1-a)$.
(6) \textbf{PCF with analytic solution (PCF-Ana): } In synthetic experiments, instead of training via ERM, we can directly acquire bayes optimal $\phi$ in closed-form. 
Detailed can be found in \Cref{app-sec:ana-solution}.
(7) \textbf{PCF with CRM (PCF-CRM): } As discussed in \Cref{sec:practical-alg}, it could be hard to get the optimal predictor when there is counterfactual estimation error. 
Here due to the scale of our experiment, we augment the dataset with estimated counterfactuals rather than finetuning.
Specifically, $\phi$ is trained via ERM on the dataset $\mathcal{D}_\train = \{x^{(i)},y^{(i)},a^{(i)}\}_{i=1}^N \cup \{\hat{x}^{(i)}_{1-a},y^{(i)},1-a^{(i)}\}_{i=1}^N $.

\subsection{Synthetic Dataset}
In this section, we consider two regression synthetic datasets and two classification tasks where all of our assumptions in \Cref{asm:main} are satisfied.
The regression tasks are as below
\begin{equation*}
\begin{aligned}[c]
    & \textit{Linear-Reg}\\
    A &\sim \textnormal{Bernoulli}(p_A), U \sim \mathcal{N}(0,1), \epsilon_Y \sim \mathcal{N}(0,1)\\
    X & =  w_A A + w_U U\\
    Y &  = w_X X + w'_U U + w_Y \epsilon_Y\\
\end{aligned}
\quad 
\begin{aligned}[c]
    & \textit{Cubic-Reg}\\
    A &\sim \textnormal{Bernoulli}(p_A), U \sim \mathcal{N}(0,1), \epsilon_Y \sim \mathcal{N}(0,1)\\
    X & =  w_A A + w_U U\\
    Y &  = w_X X^3 + w'_U U + w_Y \epsilon_Y\\
\end{aligned}
\end{equation*}
The classification tasks take the same form except $Y \sim \textnormal{Bernoulli}( \sigma(w_X X + w'_U U + w_Y \epsilon_Y )$ and $Y  \sim \textnormal{Bernoulli}( \sigma(w_X X^3 + w'_U U + w_Y \epsilon_Y ))$ for \textit{Linear-Cls} and \textit{Cubic-Cls} respectively.
More details could be found in \Cref{app-sec:dataset}.
Results are averaged over 5 different runs where the structural model is kept the same but data is resampled.
All results shown in the main paper use KNN based predictor. 
Results with other predictors can be found in \Cref{app-sec:exp-result}.

\paragraph{Optimality of PCF given true counterfactuals}
We first test different methods in situations where all methods have access to ground truth counterfactuals and $U$ as needed.
In \Cref{fig:gt_knn}, we observe that while CFE, CFR and PCF all achieve perfect CF, PCF has lowest predictor error. 
This validates \Cref{thm:opt-predictor} regarding the optimality of PCF under the constraint of CF.
Furthermore, since here ERM can get solution close to optimal predictor (this indicates the plugin $\phi$ used by PCF is also close to being optimal), we can also observe the inherent fairness-utility trade-off discussed in \Cref{thm:excess-risk-cf}.

\begin{figure}[!ht]
    \centering
    \begin{subfigure}[t]{0.22\linewidth}
        \centering
        \includegraphics[width=\linewidth]{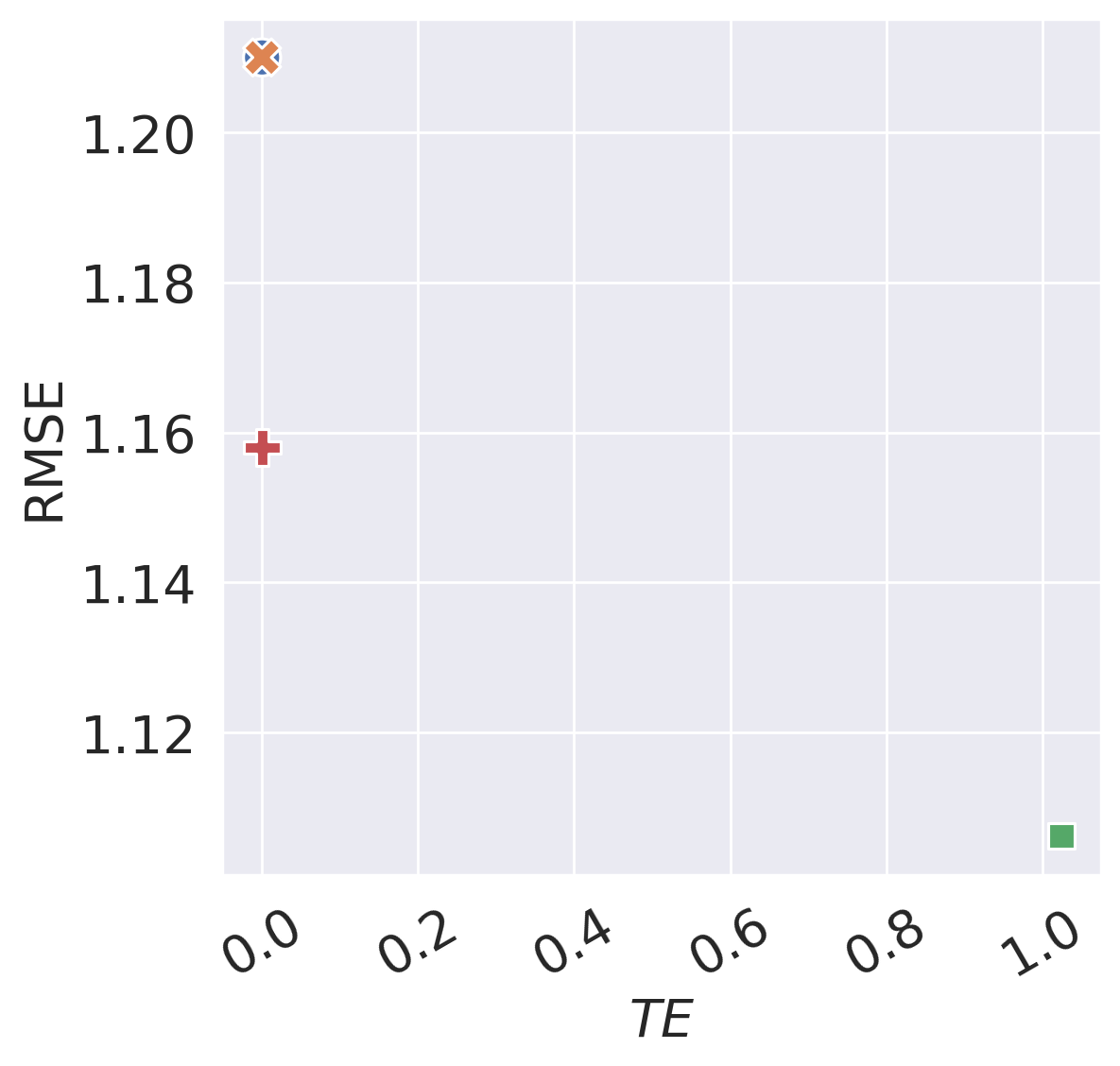}
        \caption{\textit{Linear-Reg}}
    \end{subfigure}
    \begin{subfigure}[t]{0.22\linewidth}
        \centering
        \includegraphics[width=\linewidth]{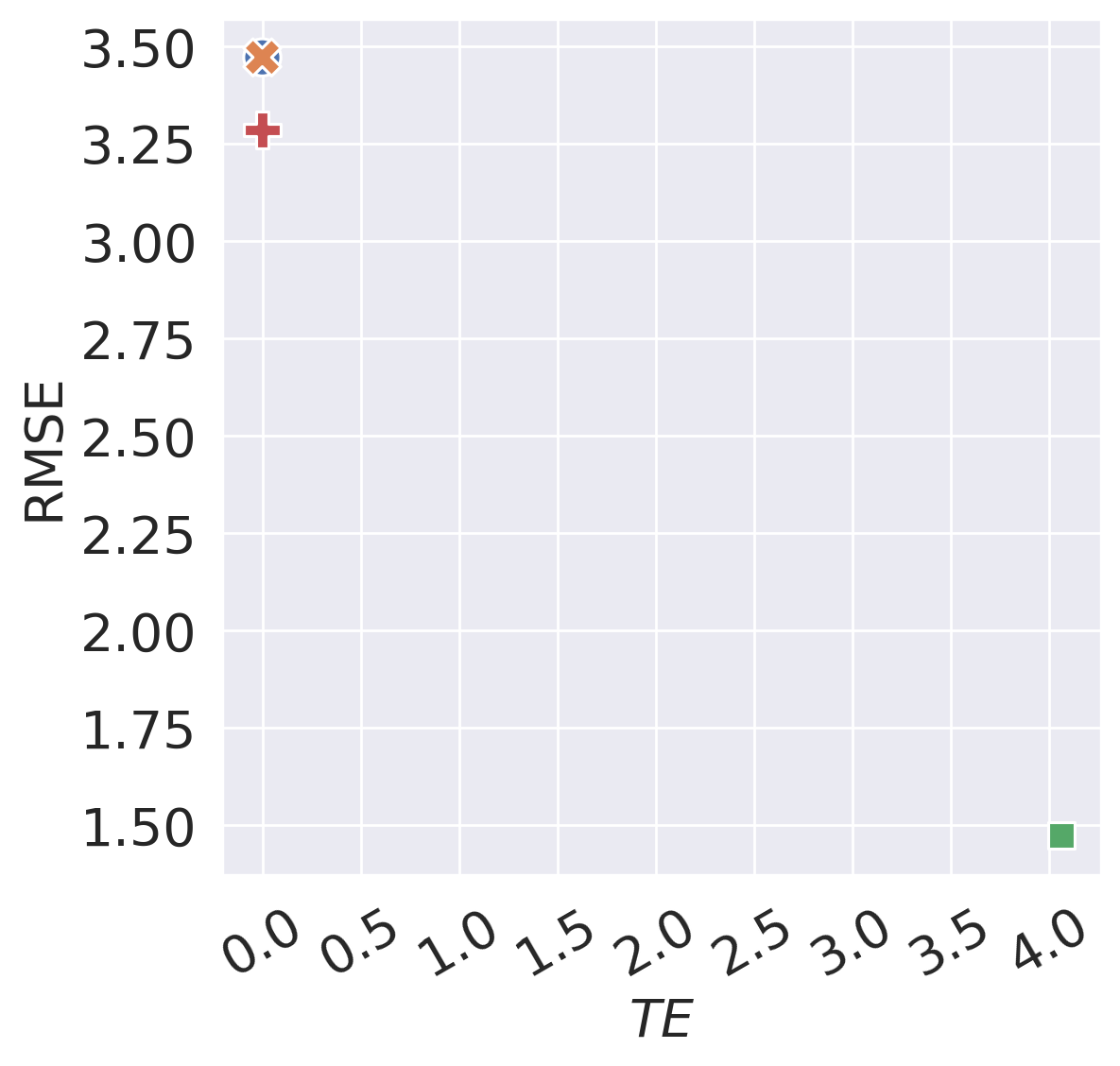}
        \caption{\textit{Cubic-Reg}}
    \end{subfigure}
    \begin{subfigure}[t]{0.22\linewidth}
        \centering
        \includegraphics[width=\linewidth]{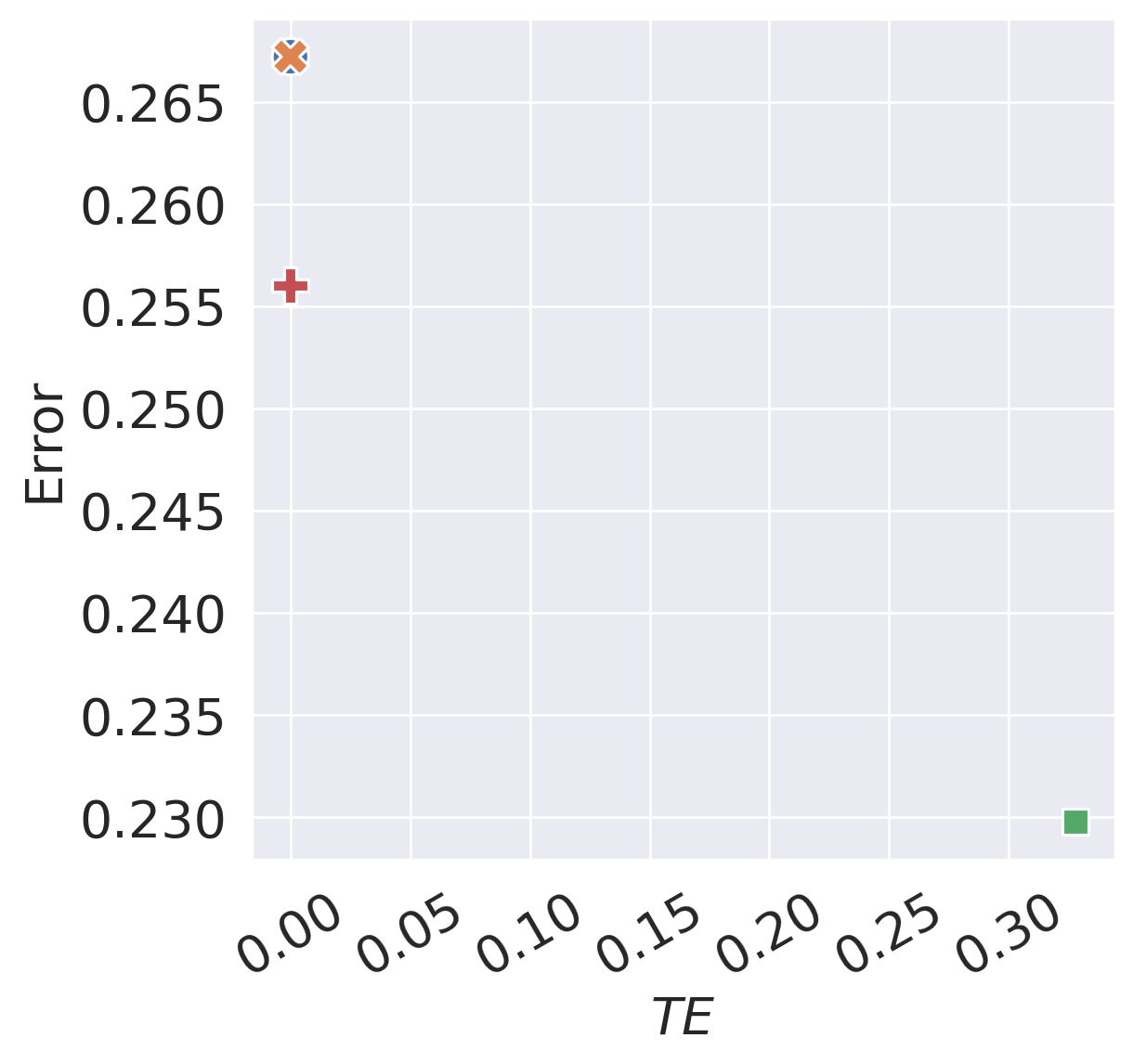}
        \caption{\textit{Linear-Cls}}
    \end{subfigure}
    \begin{subfigure}[t]{0.27\linewidth}
        \centering
        \includegraphics[width=\linewidth]{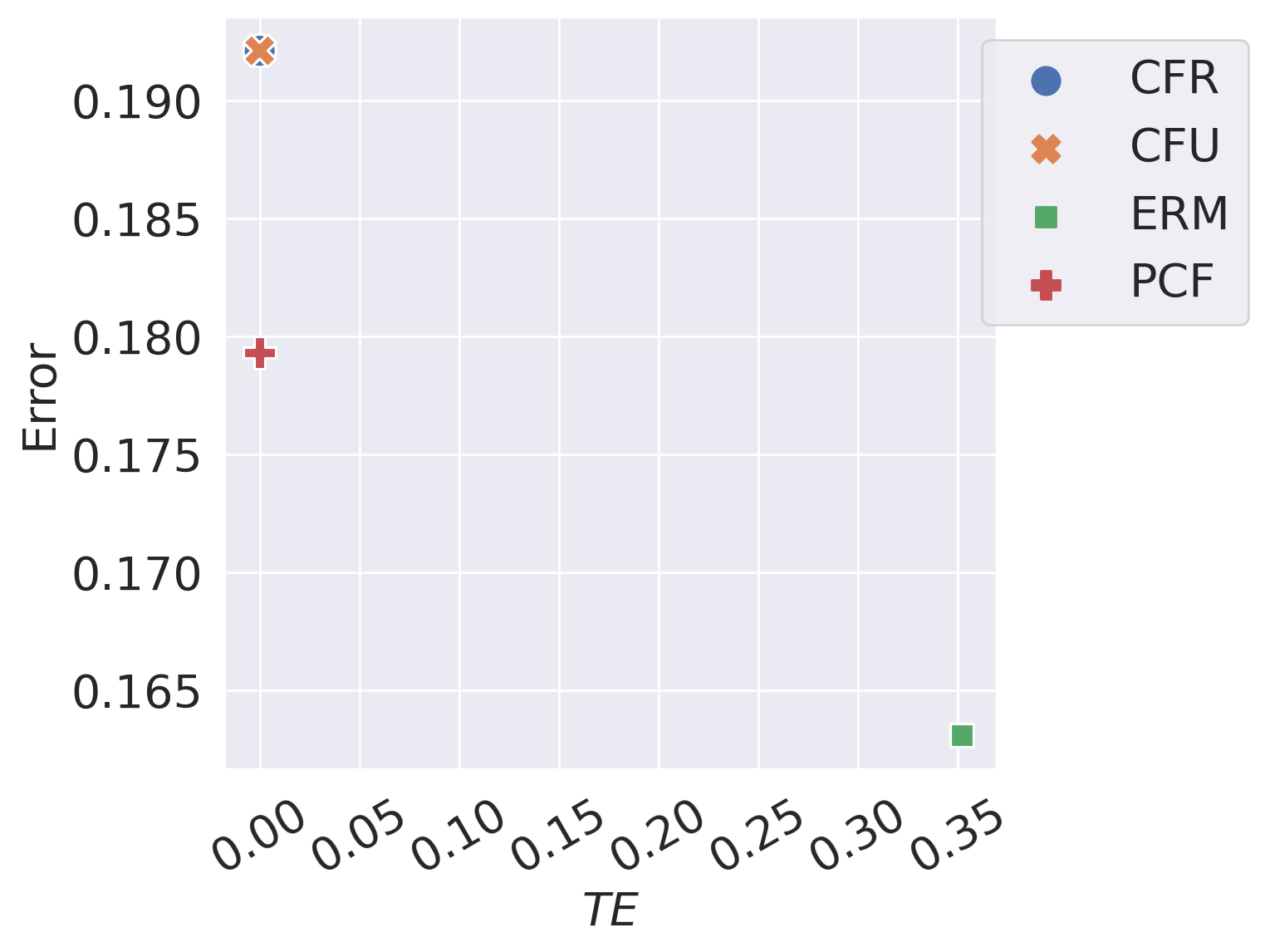}
        \caption{\textit{Cubic-Cls}}
    \end{subfigure}
    \caption{Results on synthetic datasets given ground truth counterfactuals.}
    \label{fig:gt_knn}
\end{figure}

\begin{figure}[!ht]
    \centering
    \begin{subfigure}[t]{0.22\linewidth}
        \centering
        \includegraphics[width=\linewidth]{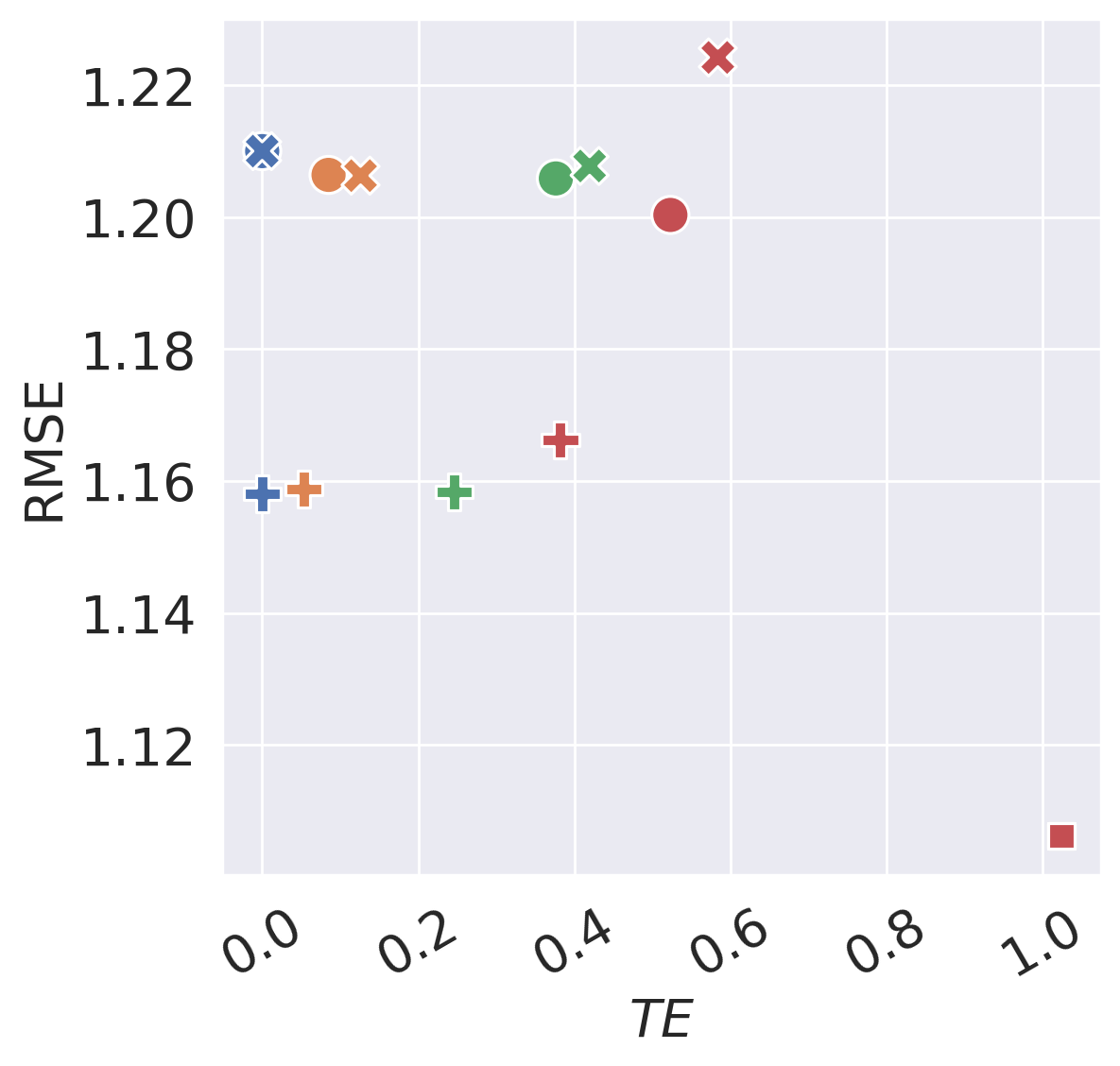}
        \caption{\textit{Linear-Reg}}
    \end{subfigure}
    \begin{subfigure}[t]{0.22\linewidth}
        \centering
        \includegraphics[width=\linewidth]{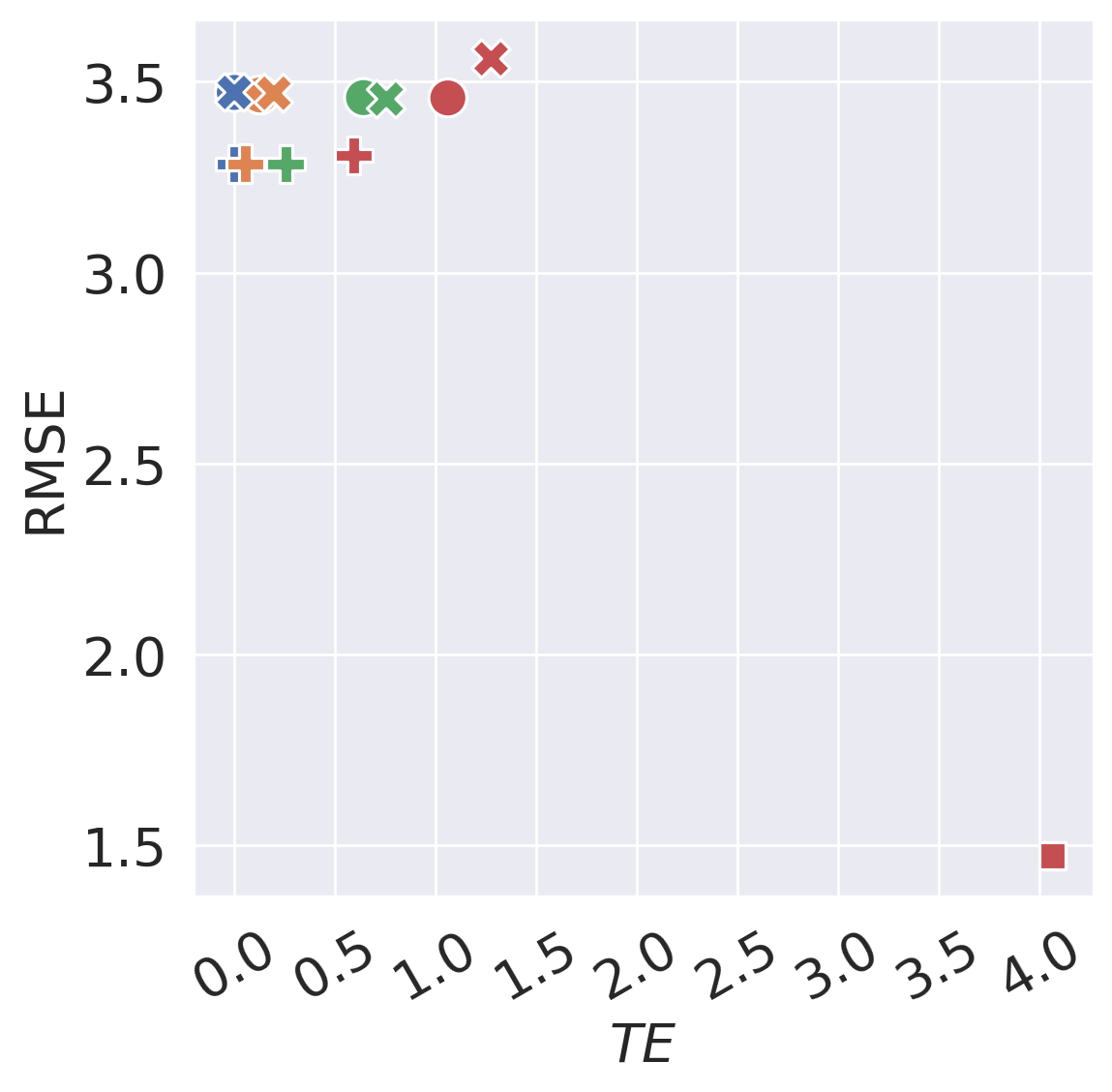}
        \caption{\textit{Cubic-Reg}}
    \end{subfigure}
    \begin{subfigure}[t]{0.22\linewidth}
        \centering
        \includegraphics[width=\linewidth]{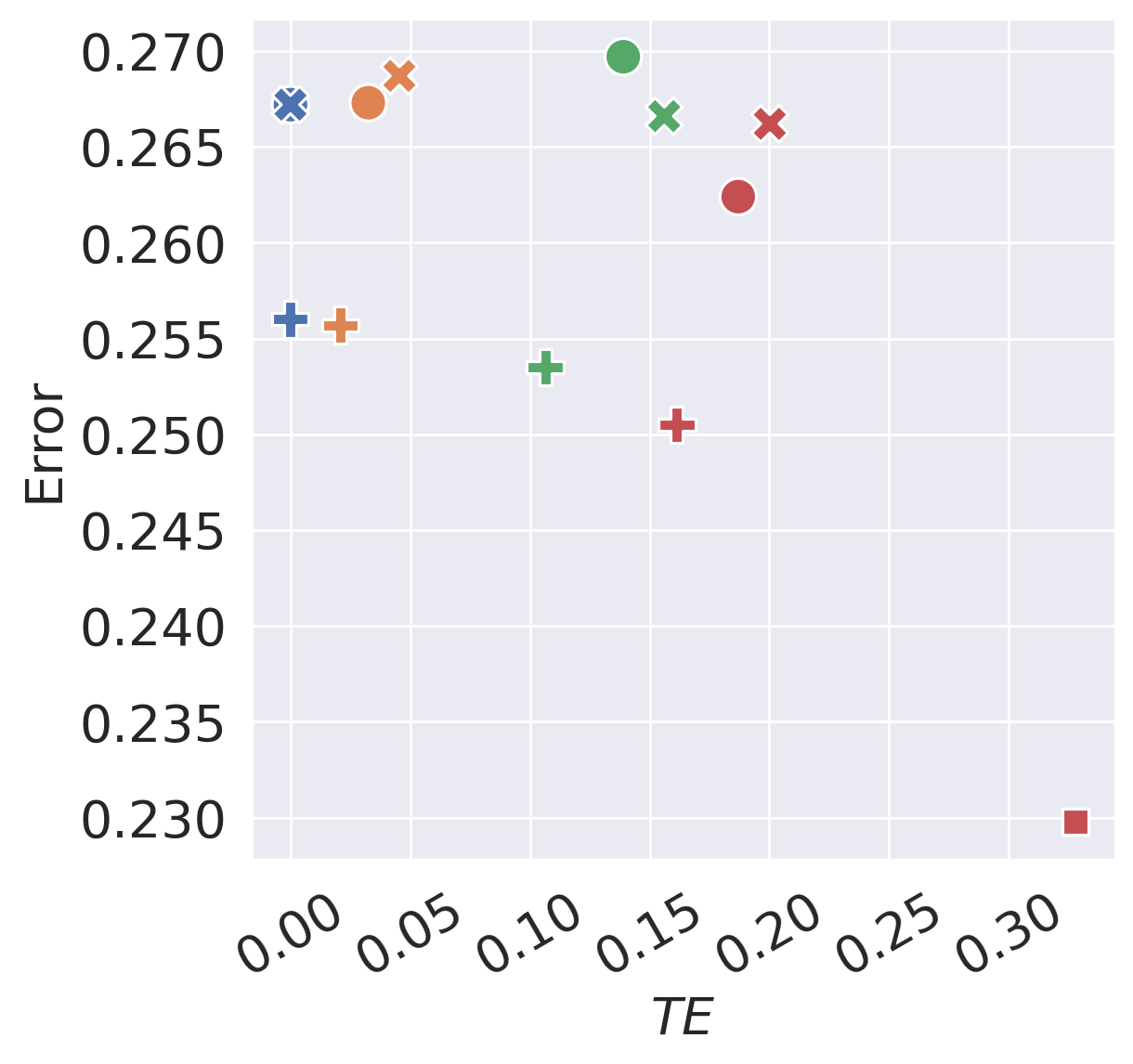}
        \caption{\textit{Linear-Cls}}
    \end{subfigure}
    \begin{subfigure}[t]{0.28\linewidth}
        \centering
        \includegraphics[width=\linewidth]{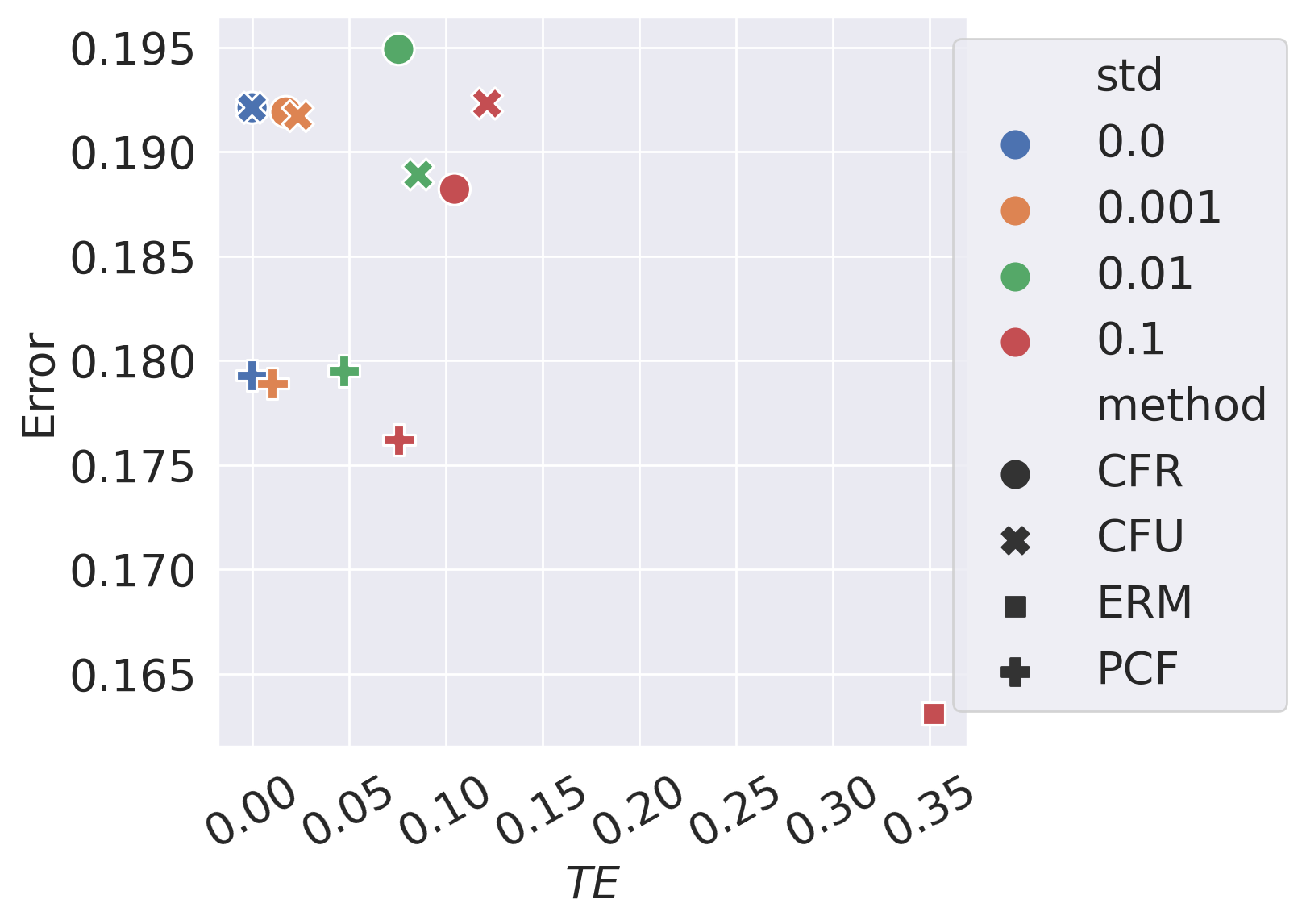}
        \caption{\textit{Cubic-Cls}}
    \end{subfigure}
    \caption{Results on synthetic datasets under counterfactual estimation error. 
    Different color represents different $\alpha$ indicating the standard deviation of the error ($\epsilon \sim \mathcal{N}(0,\alpha)$) while shape represents different algorithms.
    Results with different $\beta$ can be found in \Cref{app-sec:exp-result}.}
    \label{fig:est_knn}
\end{figure}

\begin{figure}[!ht]
    \centering
    \begin{subfigure}[t]{0.22\linewidth}
        \centering
        \includegraphics[width=\linewidth]{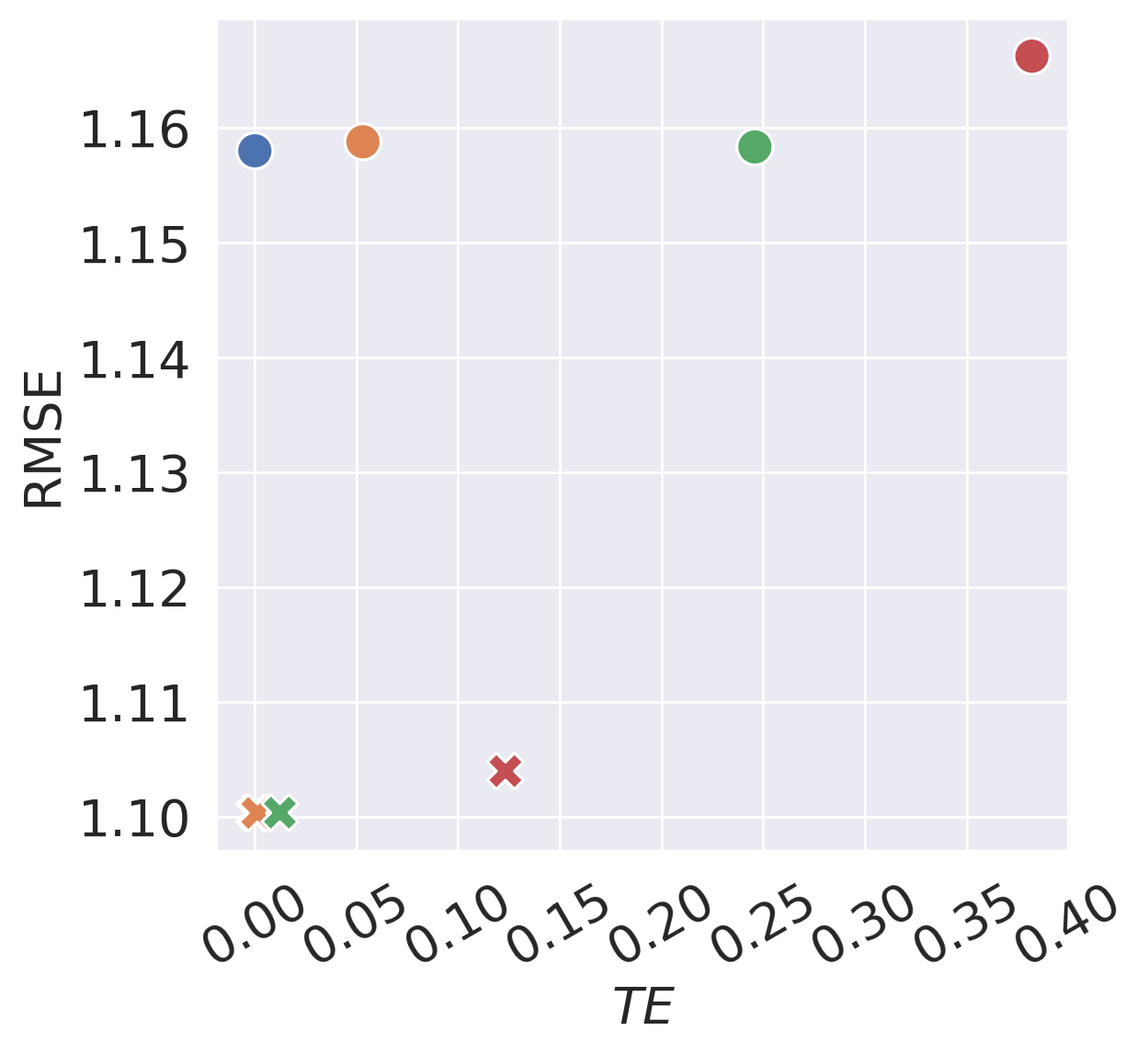}
        \caption{\textit{Linear-Reg}}
    \end{subfigure}
    \begin{subfigure}[t]{0.22\linewidth}
        \centering
        \includegraphics[width=\linewidth]{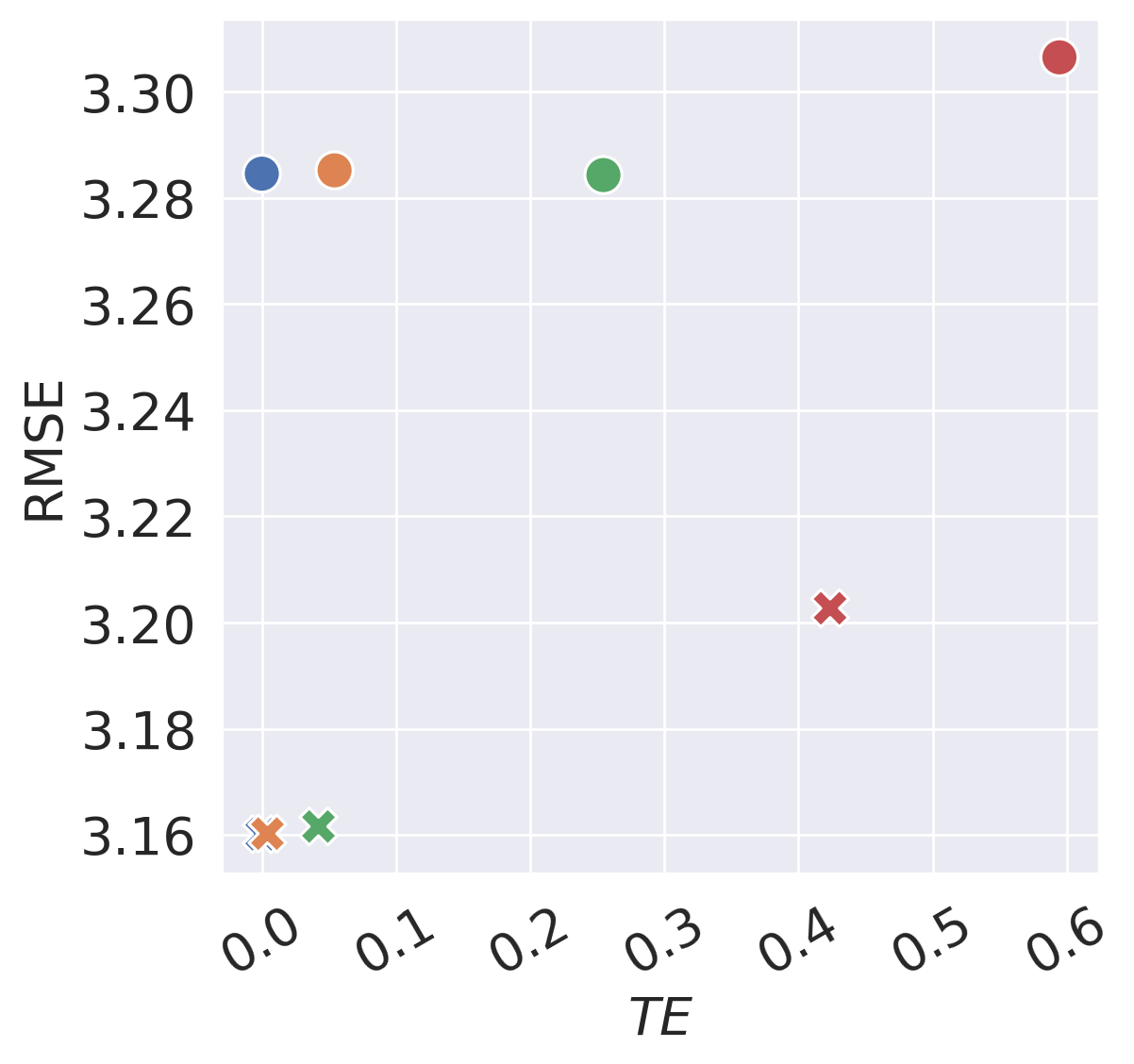}
        \caption{\textit{Cubic-Reg}}
    \end{subfigure}
    \begin{subfigure}[t]{0.22\linewidth}
        \centering
        \includegraphics[width=\linewidth]{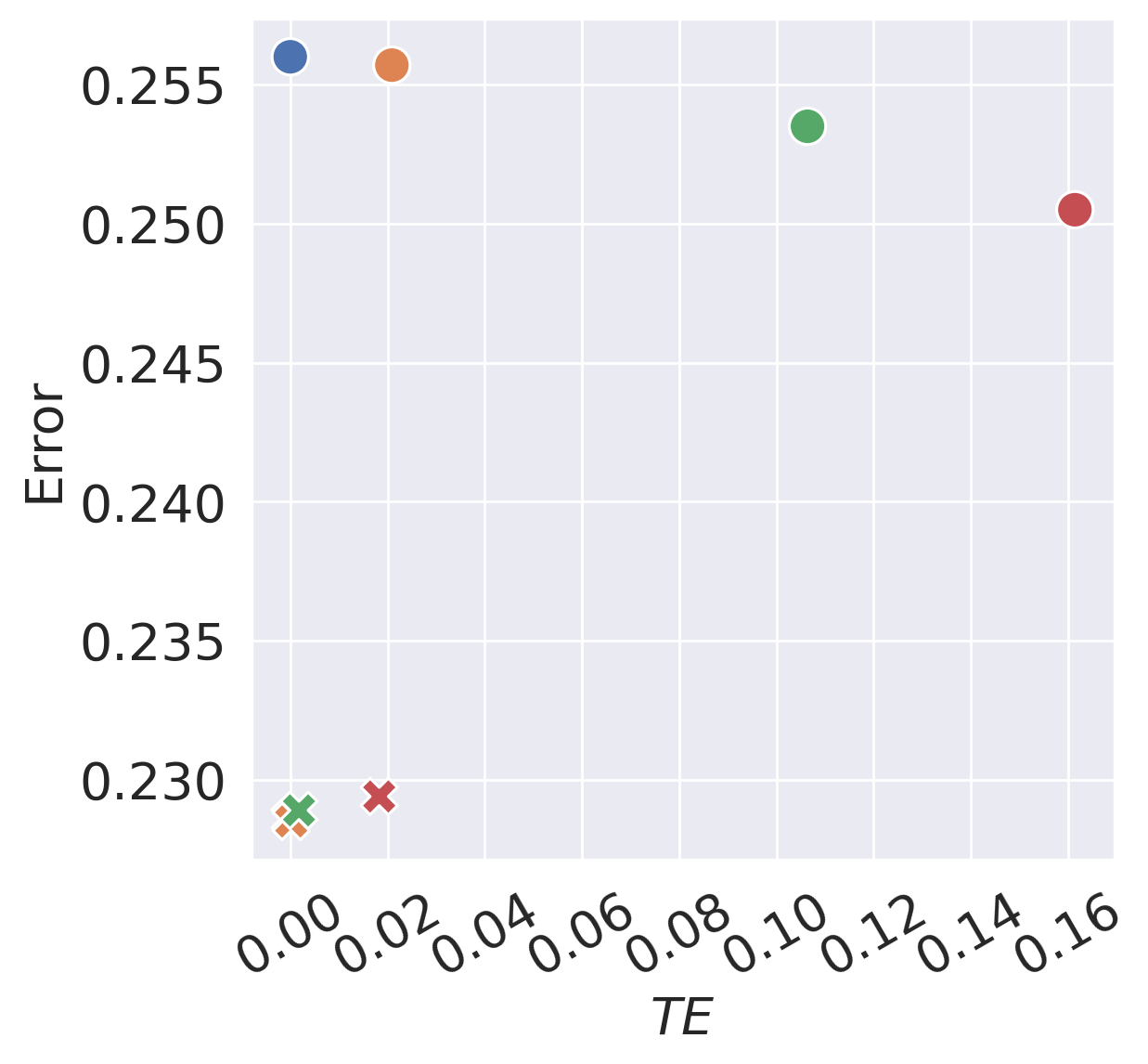}
        \caption{\textit{Linear-Cls}}
    \end{subfigure}
    \begin{subfigure}[t]{0.28\linewidth}
        \centering
        \includegraphics[width=\linewidth]{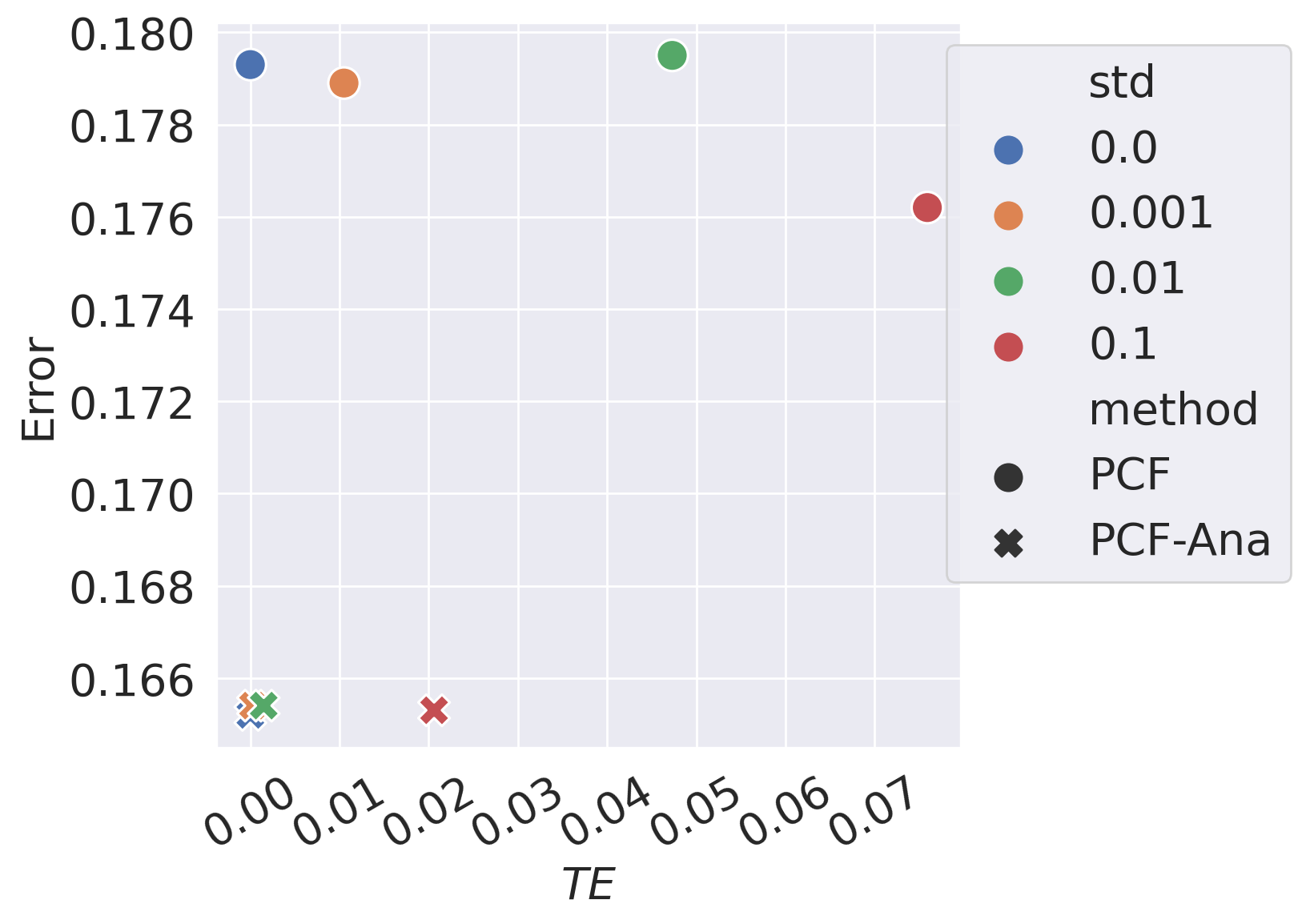}
        \caption{\textit{Cubic-Cls}}
    \end{subfigure}
    \caption{Results on synthetic datasets comparing PCF and PCF-Analytic.
    Different color represents different $\alpha$ indicating the standard deviation of the error ($\epsilon \sim \mathcal{N}(0,\alpha)$) while shape represents different algorithms.
    Results with different $\beta$ can be found in \Cref{app-sec:exp-result}.}
    \label{fig:estana_knn}
    \vspace{-1em}
\end{figure}

\paragraph{Performance under controllable error}
Here we investigate a more practical scenario where both counterfactuals and $U$ need to be estimated. 
To investigate how error and TE changes with counterfactual estimation error in a more controllable way and investigate , we simulate the estimation error by adding gaussian noise.
Specifically, $\hat{x}_{a'} = x_{a'} + \epsilon$ and $\hat{u} = u+ \epsilon$ where $\epsilon \sim \mathcal{N}(\beta,\alpha)$.
In \Cref{fig:est_knn}, we observe that while the fairness and ML performance (especially fairness) of CFE, CFR and PCF tends to get worse as error gets more significant, PCF remains best for all noise level.

\paragraph{Investigating source of error}
Here we further investigate what could be source of error in the previous scenario. 
As discussed in \Cref{sec:practical-alg-g-phi}, in practice, two things in \Cref{thm:opt-predictor} break down: access to Bayes optimal classifier and ground truth counterfactuals.
In \Cref{fig:estana_knn}, we observe that PCF-Analytic tends to be more robust against counterfactual estimation error than PCF.
We argue this is because $\phi$ used in PCF is not trained well on the estimated counterfactual distribution.

\begin{figure}
    \centering
        \includegraphics[width=0.55\linewidth]{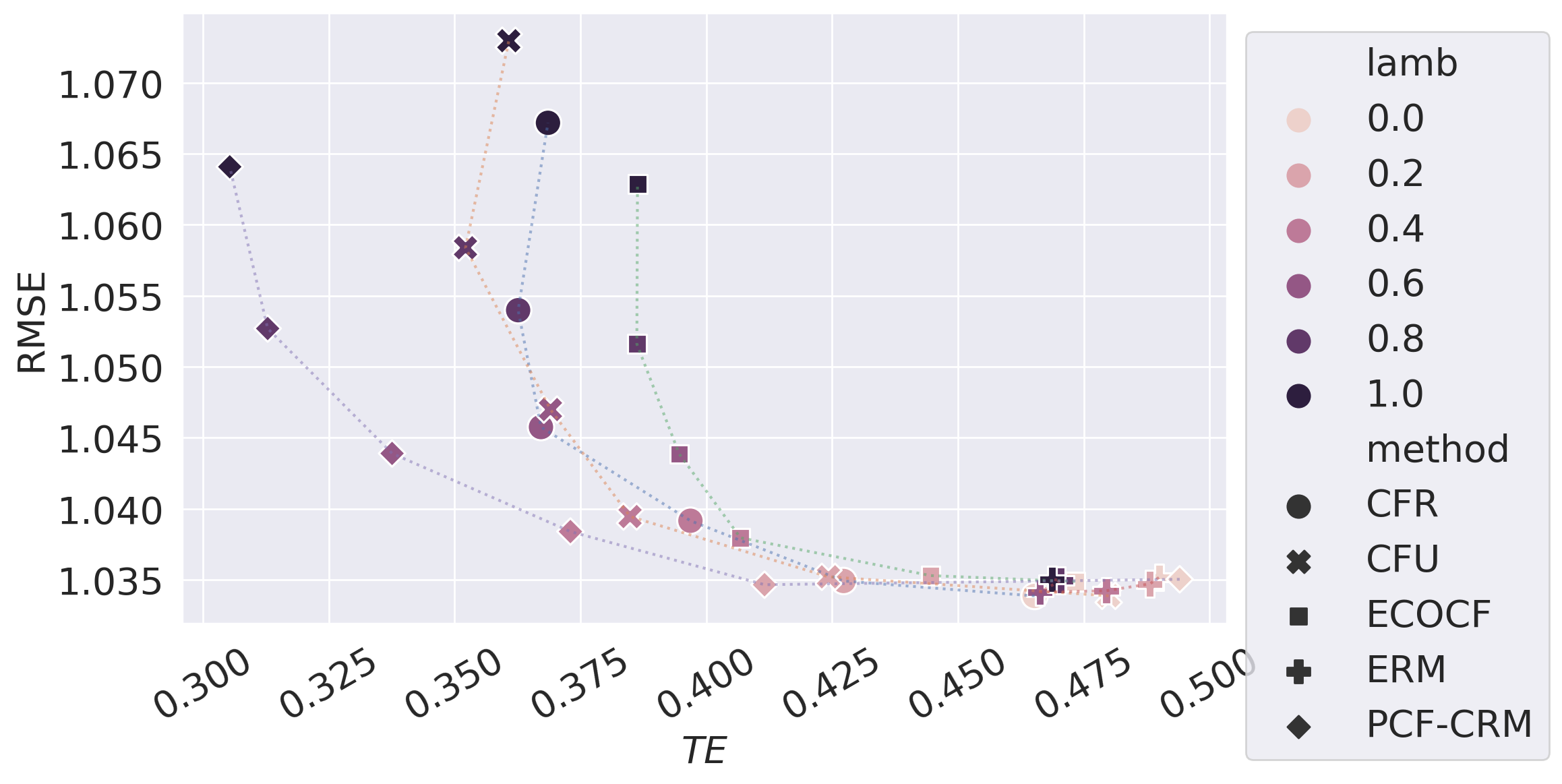}
    \caption{Results on Sim-Law with estimated counterfatuals.
    The predictor is a MLP regressor.
    We also test the convex combination of each algorithm and ERM.
    For example, PCF-CRM with $\lambda$ means $\hat{y} = \lambda \hat{y}_{\textnormal{PCF-CRM}} + (1-\lambda) \hat{y}_{\textnormal{ERM}}$.
    This suggests that PCF-CRM can achieve lower Error given the same TE and lower TE given the same Error.}
\label{fig:law}
\end{figure}

\subsection{Semi-synthetic Dataset}
In this section, we consider Law School Success dataset \citep{wightman1998lsac} where the sensitive attribute is gender and the target is first-year grade..
The main goal of this experiment is to validate the effectiveness of our methods in more practical scenarios where limited causal knowledge is available and the invertibility assumption is relaxed.

To compute TE, we need access to ground truth counterfactuals.
Hence we train a generative model on real dataset to generate semi-synthetic dataset following the method in \citet{zuo2023counterfactually}.
We want to emphasize that counterfactuals are hidden from downstream models and used for the evaluation of TE only.
This way, we get access to the ground truth $u^*$ and can generate ground truth counterfactuals without any error.
In our investigation, exogenous noise, factual data and counterfactual data are all actually the simulated version of original datasets.
However they do follow a fixed data generating mechanism that is close to the real data.
More details could be found in \Cref{app-sec:dataset}. 
All experiments are repeated 5 times on the same semi-synthetic dataset.

\paragraph{Results}

In \Cref{fig:law}, we observe that PCF-CRM  achieves better CF and lower Error in comparison to CFU and CFR. 
This validates our improvement \Cref{sec:practical-alg-g-phi} indeed leads to more practical algorithm.
Results on comparing PCF and PCF-CRM can be found in \Cref{app-sec:exp-result}, which further justifies this.
While ERM could achieve lower error, it has worst fairness.
This is inevitably determined by the inherent trade-off discussed in \Cref{thm:excess-risk-cf}.
Furthermore, inspired by the trade-off, we test the result of mixing all predictors with ERM. 
The curve shows that PCF-CRM remains optimal given fixed CF and best CF given fixed error.
This demonstrates again that PCF-CRM is the best among all methods.

\section{Conclusion and Discussion}\label{sec:conclusion}

\paragraph{Conclusion}

In this work, we conducted a formal investigation of the trade-off between Counterfactual Fairness (CF) and predictive performance. 
We proved that combining factual and counterfactual predictions with a potentially unfair, optimal predictor achieves optimal CF. 
Additionally, we derived the excess risk between predictors with and without CF constraints, quantifying the minimum performance degradation necessary to ensure perfect CF.
To address incomplete causal knowledge, we analyzed the effects of imperfect counterfactual estimations on CF and predictive performance. 
We proposed a plugin approach that leverages pre-trained models for optimal fair prediction and developed a practical method to mitigate estimation errors.

Despite the theoretical contributions of our method, two limitations may impact practical applicability: (1) access to ground-truth counterfactuals and (2) access to Bayes optimal predictors. 
Below, we delve into these limitations, clarifying how our methods can be practically applied and how they can benefit from contributions from the broader community. 
We hope this discussion will also inspire future research directions.

\paragraph{Access to ground truth counterfactuals}

While how to better estimate counterfactuals is out of the scope of this work, it is indeed an unavoidable challenge faced by the community of Counterfactual Fairness.
It not only limits the deployment of CF algorithms, but also leads to difficulty in validating proposed CF methods.
While counterfactual data can be obtained in specific scenarios, such as through randomized controlled trials, it is challenging to acquire in most applications.
There are some works in the field of causality that aims at estimating counterfactuals.
For instance, \citet{nasr2023counterfactual} proves counterfactual identifiability under certain causal graphs.
However, in more general scenarios, such causal knowledge may be lacking and identifying the causal graph itself can be challenging. 
These tasks have been well studied in the field of causal discovery \citep{chickering2002optimal, colombo2014order} and causal representation learning \citep{scholkopf2021toward}.
Solutions to this problem typically rely on strong assumptions, such as the linearity of Structural Causal Models (SCMs) or additive noise \citep{shimizu2006linear, hoyer2008nonlinear, peters2014causal}.
More recently, \citet{kulinski2023towards} propose a method of estimating counterfactuals without the need to identify the causal model or graph.
We believe this approach to direct counterfactual estimation could have the potential to be a good plugin counnterfactual estimator in our algorithm.
Additionally, generative models could also be used to generate samples as if they had come from a different sensitive attribute \citep{choi2018stargan,zhou2022iterative,zhou2023efficient,rombach2022high}.
These methods often offer the advantage of higher sample quality, especially in modalities such as images or natural language. 
However, they must be applied with considerable care, as they generally lack integration with the causal model and may introduce significant estimation errors.

\paragraph{Access to Bayes optimal predictors}
Another crucial plugin estimator of our method is the optimal predictor. 
In classical ML settings, achieving a good estimator for the counterfactual distribution often requires retraining or fine-tuning. 
However, in this era, with the abundance of pre-trained models, such as foundation models \citep{bommasani2021opportunities}, it could be much easier to get a predictor that is close to being optimal.
Rather, given that these models are trained on noisy internet data and have extensive reach and impact, it is of great importance to find effective ways to debias them. 
We propose that our plugin algorithm could be a suitable solution due to its post-processing nature, which avoids incurring significant computational costs.

\section*{Acknowledgement}
Z.Z., R.B., and D.I. acknowledge support from NSF (IIS-2212097), ARL (W911NF-2020221), and ONR (N00014-23-C-1016). 
M.K. acknowledges support from NSF CAREER 2239375, IIS 2348717, Amazon Research Award and Adobe Research.
T.L. and J.G. acknowledge support from NSF-IIS2226108.
Any opinions, findings, and conclusions or recommendations expressed in this material are those of the author(s) and do not necessarily reflect the views of any funding source.

\bibliography{reference}
\bibliographystyle{plainnat}

\clearpage
\appendix

\section{Proofs} \label{app-sec:proof}

\subsection{Proof of \Cref{thm:perfect-te}}
\begin{proof}[Proof of \Cref{thm:perfect-te}]
    \begin{align}
        \mathrm{TE}(\phi) =0 
        \Leftrightarrow \E[|\phi(X,A)-\phi(X_{1-A},1-A)|] =0 \Leftrightarrow \phi(x,a)\overset{\text{a.s.}}{=}\phi(x_{1-a},1-a),\quad\forall (x,a) \,,
    \end{align}
    where the first equality is by definition and the second equality is because absolute value is always non-negative for any $(x,a)$. Thus, the predictions must be almost surely equal for all $(x,a)$. Similarly, if they are all equal on the non-zero metric set, then the expectation must be 0.
\end{proof}

\subsection{Proof of \Cref{thm:opt-predictor}} \label{app-sec:proof-opt-predictor}
Before proving the main theorem, we first provide one well-known lemma that  reminds the reader of the well-known result of the optimal predictor, which is denoted by $\phi^*$ in the theorem statement.

\begin{lemma}[Optimal Predictor is Conditional Mean]
    The conditional mean $\E[Y|X=x]$ is the optimal predictor without fairness constraints for classification with cross-entropy loss and for regression with MSE loss.
\end{lemma}
\begin{proof}
    First, let's establish that the optimal predictor without constraints is in fact $\E[Y|X=x]$.
    For squared $L_2$ loss, we have that derivative:
    \begin{align*}
        &\E[\ell(\phi(X), Y)]\\ 
        &=\E_X[\E_{Y|X}[(Y - \phi(X))^2]] \\
        &=\E_X[\E_{Y|X}[Y^2] - 2\E_{Y|X}[Y\phi(X)] + \E_{Y|X}[\phi(X)^2]] \\
        &=\E_X[\E_{Y|X}[Y^2] - 2\phi(X)\E_{Y|X}[Y] + \phi(X)^2\E_{Y|X}[1]] \\
        &=\E_X[\E_Y[Y^2] - 2\phi(X)\E[Y|X] + \phi(X)^2]
    \end{align*}
    Taking the derivative of the inside expectation w.r.t. $\phi(X)$ and setting to 0 yields $\phi^*(X) = \E[Y|X]$.

    Now let's look at cross-entropy loss for classification:
    \begin{align*}
        &\E[\ell(\phi(X), Y)] \\
        &=\E_X[\E_{Y|X}[-Y\log (\phi(X))- (1-Y) \log (1-\phi(X))]] \\
        &=\E_X[-\log (\phi(X))\E[Y|X]- \log (1-\phi(X))\E[(1-Y)|X] ]] 
    \end{align*}
    Again, if you take the derivative w.r.t. $\phi(X)$ and set to 0, we see that $\phi^*(X) = \E[Y|X]$.
\end{proof}

Now we seek to prove \Cref{thm:opt-predictor}.
\begin{proof}
    First, we decompose the factual error across the sensitive attribute $A$ given the exogenous noise $U$.
    \begin{align*}
        &\E_{X,A,Y}[\ell(\phi(X,A), Y)] \\
        &=\E_{U,A,Y}[\ell(\phi(F^*_X(U,A),A), Y)] \\
        &=\E_U[\E_A[\E_{Y|U,A}[\ell(\phi(F^*_X(U,A),A), Y)]]] \\
        &=\E_U[p(A=a)\E_{Y|U,A=a}[\ell(\phi(F^*_X(U,a),a), Y)] + p(A=1-a)\E_{Y|U,A=1-a}[\ell(\phi(F^*_X(U,1-a),1-a), Y)]] \,.
    \end{align*}
    Consider $U=u,$ inside the expectation we have
    \begin{align*}
        &p(A=a)\E_{Y|U=u,A=a}[\ell(\phi(F^*_X(u,a),a), Y)] + p(A=1-a)\E_{Y|U=u,A=1-a}[\ell(\phi(F^*_X(u,1-a),1-a), Y)]\\
        &=p(A=a)\E_{Y|X=x,A=a}[\ell(\phi(x,a), Y)] + p(A=1-a)\E_{Y|X=x_{1-a},A=1-a}[\ell(\phi(x_{1-a},1-a), Y)] \,,
    \end{align*}
where w.l.o.g., $x$ is viewed as the factual and $x_{a'}$ is viewed as the counterfactual.
Because of invertibility, these two terms are unique for every $(u,a)$ or correspondingly $(x,a)$ combination and thus the problem decomposes across $U$.
Thus, the factual loss can be viewed as a combination of the factual loss from one specific $a$ plus the counterfactual loss for $a'$ for each point $x$.
    
We have the following subproblems indexed by $u$: The factual loss can be viewed as a combination of the factual loss from one specific $a$ plus the counterfactual loss for $1-a$ for each point $x$. Notice that the constraint is $\phi(x,a)\overset{a.s.}{=} \phi(x_{1-a},1-a)$ from \Cref{thm:perfect-te}. We can directly push the constraint into the optimization problem by optimizing over $\phi_0 \triangleq \phi(x,a) \overset{a.s.}{=} \phi(x_{1-a},1-a)$:
\begin{align}
    \argmin_{\phi} p(A=a)\E_{Y|X=x,A=a}[\ell(\phi_0, Y)] + p(A=1-a)\E_{Y|X=x_{1-a},A=1-a}[\ell(\phi_0, Y)]
\end{align}
Taking $\ell$ as squared $L_2$ loss: we have 
    \begin{align*}
        &\argmin_{\phi} p(A=a)\E_{Y|X=x,A=a}[(Y-\phi_0)^2] + p(A=1-a)\E_{Y|X=x_{1-a},A=1-a}[(Y-\phi_0)^2] \\
        &=\argmin_{\phi}p(A=a)\{\E_{Y|X=x,A=a}[Y^2] -2\phi_0\E_{Y|X=x,A=a}[Y] + \phi_0^2\}\\
        & + p(A=1-a)\{\E_{Y|X=x_{1-a},A=1-a}[Y^2] -2\phi_0\E_{Y|X=x_{1-a},A=1-a}[Y] + \phi_0^2\} \,.
    \end{align*}
Similarly, if we take $\ell$ as (binary) cross-entropy loss: we have
\begin{align*}
    &\argmin_{\phi} p(A=a)\E_{Y|X=x,A=a}[-(Y\log(\phi)+(1-Y)\log(1-\phi))] \\
    &+ p(A=1-a)\E_{Y|X=x_{1-a},A=1-a}[-(Y\log(\phi)+(1-Y)\log(1-\phi))] 
\end{align*}

It is simple to see that both loss functions are convex, thus could obtain a unique solution by taking the derivative. Thus, for each $x, x_{a'}$ induced by $U=u$, we could get the optimal $\phi_0$:
\begin{align*}
\phi_0=\sum_{a' \in \{0,1\}}p(A=a')\phi^*(x_{a'},a')\,,
\end{align*}
where $\phi^*$ is the optimal predictor from the lemma above.
This result holds for every $u$ and thus gives the final result.
\end{proof}

\subsection{Proof of \Cref{thm:excess-risk-cf}}

\begin{proof}
Let $\phi^*(x,a)$ and $\phi^*_{\textnormal{CF}}(x,a)$ be the Bayes optimal predictor under no constraint and CF constraint respectively.
We have shown that $\phi^*(x,a) = \E_{Y|X=x,A=a}[Y]$ and  $\phi^*_{\textnormal{CF}}(x,a) =p(A=a)\phi^*(x,a) + p(A=1-a)\phi^*(x_{1-a}, 1-a) $.

Noting that $\phi^*$ is Bayes optimal, its risk satisfies $\risk^* \leq \risk^*_\CF$ where $\risk^*_\CF$ denotes the risk of  $\phi^*_{\textnormal{CF}}(x,a)$. 
By definition, the excess risk of $\phi^*_{\textnormal{CF}}(x,a)$ is 
\begin{align*}
   \risk^*_\CF - \risk^*  
    =& \mathbb{E}_{X,A}\left[\E_{Y|X=x,A=a}[\ell(\phi^*_\CF(x,a), Y) -\ell(\phi^*(x,a), Y) ]\right].
\end{align*}

For regression task and real-valued $Y$, we take $\ell$ as squared $L_2$ loss and have 
\begin{align*}
& \risk^*_\CF - \risk^*  \\
=& 
\mathbb{E}_{X,A}\left[\E_{Y|X=x,A=a}[(\phi^*_\CF(x,a)-Y)^2 -(\phi^*(x,a)-Y)^2 ]\right]\\
=& 
\mathbb{E}_{X,A}\left[\E_{Y|X=x,A=a}[(\phi^*_\CF(x,a)-\phi^*(x,a))(\phi^*_\CF(x,a)+\phi^*(x,a) -2Y) ]\right]\\
=& 
\mathbb{E}_{X,A} \left[(\phi^*_\CF(x,a)-\phi^*(x,a))\E_{Y|X=x,A=a}[(\phi^*_\CF(x,a)+\phi^*(x,a) -2Y) ]\right]\\
=& 
\mathbb{E}_{X,A} \left[(\phi^*_\CF(x,a)-\phi^*(x,a))^2 \right] \\
=& 
\mathbb{E}_{X,A} \left[ ((p(A=a)\phi^*(x,a) + p(A=1-a)\phi^*(x_{1-a}, 1-a) -\phi^*(x,a))^2 \right]\\
=& 
\mathbb{E}_{X,A} \left[ (1-p(A=a))^2 (\phi^*(x,a)  -\phi^*(x_{1-a}, 1-a))^2 \right]  \\
=& \mathbb{E}_{X,A} \left[ p^2(A=1-a) (\phi^*(x,a)  -\phi^*(x_{1-a},1-a))^2 \right]\\
=& \mathbb{E}_{A} \left[ p^2(A=1-a) \E_{X|A=a}[(\phi^*(x,a)  -\phi^*(x_{1-a},1-a))^2] \right].
\end{align*}
Let's define 
\begin{align*}
\Delta_a 
&\triangleq 
\E_{X|A=a}\left [\left(\E_{Y\mid X=x, A=a} [Y] - \E_{Y\mid X={x_{1-a}}, A=1-a} [Y] \right)^2\right] \\
&\overset{(a)}{=} 
\E_{U|A=a}\left [\left(\E_{Y\mid U=u, A=a} [Y] - \E_{Y\mid U=u, A=1-a} [Y] \right)^2\right] \\
&\overset{(b)}{=} 
\E_{U}\left [\left(\E_{Y\mid U=u, A=a} [Y] - \E_{Y\mid U=u, A=1-a} [Y] \right)^2\right] = \Delta_{1-a},
\end{align*}

where $(a)$ holds from the invertibility between $X$ and $U$ given $A$, and $(b)$ holds from the fact that $U$ and $A$ are independent. 
For simplicity, we denote $\Delta = \Delta_{a} = \Delta_{1-a}$. 
Notably, $\Delta$ measures the expected change of $Y$ due to the change of $A$ over all possible $U$, and is in fact a measure of their dependency. 
Furthermore, we have 
\begin{align*}
& \mathbb{E}_{A} \left[ p^2(A=1-a) \E_{X|A=a}[(\phi^*(x,a)  -\phi^*(x_{1-a},1-a))^2] \right] \\
=& \mathbb{E}_{A} \left[ p^2(A=1-a) \Delta_a \right] \\
=& \sum_{a} p(A=a)p^2(A=1-a) \Delta_a \\
=& p(A=0)p^2(A=1) \Delta_0 + p(A=1)p^2(A=0) \Delta_1 \\
=& p(A=0)p(A=1) \Delta \\
=&\sigma^2_A \Delta.
\end{align*}

Next, for classification task and binary $Y$ using cross-entropy loss, then 
\begin{align*}
& \risk^*_\CF - \risk^*  \\
=& 
\mathbb{E}_{X,A}\left[\E_{Y|X=x,A=a} \left[-Y\log \phi^*_\CF(x,a)-(1-Y) \log(1-\phi^*_\CF(x,a))\right] + \right.\\
\quad& \left. \E_{Y|X=x,A=a} \left[ Y\log \phi^*(x,a)+(1-Y) \log(1-\phi^*(x,a)) \right]
\right]\\
=& 
\mathbb{E}_{X,A}\left[-\phi^*\log \phi^*_\CF - (1-\phi^*) \log(1-\phi^*_\CF) + \phi^*\log \phi^* + (1-\phi^*) \log(1-\phi^*)
\right]\\
=&
\mathbb{E}_{X,A}\left[\phi^*\log \frac{\phi^*}{\phi^*_\CF} + (1-\phi^*) \log
\frac{1-\phi^*}{1-\phi^*_\CF }\right]\\
\overset{(a)}{=}& 
\mathbb{E}_{X,A}\left[D_{KL}[p(Y \mid X, A) \| \E_A [p(Y \mid X_A, A)] \right]\\
\overset{(b)}{=}& 
\mathbb{E}_{U,A}\left[D_{KL}[p(Y \mid U, A) \| \E_A [p(Y \mid U, A)] \right]\\
=&
\mathbb{E}_{U,A}\left[D_{KL}[p(Y \mid U, A) \| p(Y \mid U) \right]\\
=&
\mathbb{E}_{U} \E_A \left[D_{KL}[p(Y \mid U, A) \| p(Y \mid U) \right]\\
=& 
I(A; Y \mid U),
\end{align*}
where $(a)$ holds from noting that $\phi^*(x, a) = p(Y=1 \mid X=x, A=a)$ and $\phi^*_{CF} = \E_{A} [p(Y=1 \mid X={x_A}, A)]$, 
and $(b)$ again holds from the invertibility between $X$ and $U$ given $A$. 
\end{proof}

\subsection{Proof of \Cref{thm:pcf-perfect}} \label{app-sec:proof-pcf-perfect}
\begin{proof}

\begin{align*}
    \E[\hat{Y}|X=x,A=a] 
    &= \hat{\mu}(x,a) \\
    & = p(A=a)\phi(x_a,a) + p(A=1-a)\phi(G(x,a,1-a),1-a)\\
     & = p(A=a)\phi(G(x_{1-a},1-a,a),a) + p(A=1-a)\phi(x_{1-a},1-a)\\
    &= \hat{\mu}(x_{1-a},1-a) \\
    &= \E[\hat{Y}|X=x_{1-a},A=1-a] \,,
\end{align*}
where the middle qualities are by the properties of the invertible and ground truth CGM.
Because the factual output for the algorithm is the same as the counterfactual output, then the TE must be 0 by \Cref{thm:perfect-te}.
\end{proof}

\subsection{Proof of \Cref{thm:pcf-fair-error-bound}}

\begin{proof}
We first bound TE.

Let $x_{a\to 1-a} \triangleq G^*(x_a,a,1-a)$, we have 
    \begin{align*}
        \te =& 
        \mathbb{E}_{X,A} \left[ |\phi_\PCF(x_a,a) - \phi_\PCF(x_{a\to 1-a}, 1-a)|\right]\\
        =&  
        \mathbb{E}_{X,A} \left[ |p(A=a) \phi(x_a,a) + p(A=1-a) \phi(\hat{x}_{a\to 1-a},1-a)  \right. \\
        & \qquad  \left. - p(A=a) \phi(\hat{x}_{1-a \to a},a) - p(A=1-a) \phi(x_{a\to 1-a},1-a) | \right]\\
        =& 
        \mathbb{E}_{X,A} \left[ |p(A=a) \phi(x_a,a) + p(A=1-a) \phi(\hat{G}(x_a,a,1-a),1-a)  \right. \\
        & \qquad  \left. - p(A=a) \phi(\hat{G}(G^*(x_a,a,1-a),1-a,a) - p(A=1-a)\phi(G^*(x_a,a,1-a),1-a) | \right]\\
        \overset{(a)}{\leq}&  
        \mathbb{E}_{X,A} \left[ p(A=a) |\phi(x_a,a) -\phi(\hat{G}(G^*(x_a,a,1-a),1-a,a),a) | \right.   \\
        & \qquad \left. + p(A=1-a) |\phi(\hat{G}(x_a,a,1-a),1-a)  - \phi(G^*(x_a,a,1-a),1-a) | \right]\\
        =&  
        \mathbb{E}_{X,A} \left[ p(A=a) |\phi (G^*(G^*(x_a,a,1-a),1-a,a),a) -\phi(\hat{G}(G^*(x_a,a,1-a),1-a,a),a) |  \right. \\
        & \qquad \left. + p(A=1-a) |\phi(\hat{G}(x_a,a,1-a),1-a)  - \phi(G^*(x_a,a,1-a),1-a) | \right]\\
        \overset{(b)}{\leq}&  
        \mathbb{E}_{X,A} \left[ p(A=a) L|G^*(G^*(x_a,a,1-a),1-a,a) - \hat{G}(G^*(x_a,a,1-a),1-a,a) |  \right. \\
            & \left. \qquad  + p(A=1-a) L |\hat{G}(x_a,a,1-a) - G^*(x_a,a,1-a) | \right]\\
        \overset{(c)}{\leq}&  
        \mathbb{E}_{X,A} \left[p(A=a) L\varepsilon + p(A=1-a) L \varepsilon \right]\\
        = & L\varepsilon.
    \end{align*}
    Here $(a)$ holds by the convexity of absolute value, $(b)$ is from the L-lipschitz property of $\phi$, and
    $(c)$ is by the bound of counterfactual estimation error.

    Now we prove the bound for the error.   
    Taking $\ell$ as squared $L_2$ loss, we have 

    \newcommand{\pp}{\phi_{\text{PCF}}}
    \newcommand{\pf}{\phi^*_{\text{CF}}}
    
    \begin{align*}
    \risk 
    &= 
    \E_{X, A} \E_{Y \mid X=x, A=a} \left[ (\pp(x, a) - y)^2 \right] \\
    &= 
    \E_{X, A} \E_{Y \mid X=x, A=a} \left[ (\pp(x, a) - \pf(x, a) + \pf(x, a) - y)^2 \right],
    \end{align*}
    Taking the inner expectation and omit subscript for brevity, we have
    \begin{align*}
    &
    \E \left[ (\pp(x, a) - \pf(x, a) + \pf(x, a) - y)^2 \right] \\
    =& 
    \E \left[ (\pp(x, a) - \pf(x, a))^2 \right] + 2 \E \left[(\pp(x, a) - \pf(x, a))(\pf(x, a) - y) \right] + C \\
    =& 
    p(A=1-a)^2 (\phi^*(\hat x_{1-a}, 1-a) - \phi^*(x_{1-a}, 1-a) )^2 \\
    \quad &+ 2 p(A=1-a) (\phi^*(\hat x_{1-a}, 1-a) - \phi^*(x_{1-a}, 1-a)) (\pf(x, a) - \phi^*(x, a)) + C \\
    \leq& 
    p(A=1-a)^2 (\phi^*(\hat x_{1-a}, 1-a) - \phi^*(x_{1-a}, 1-a) )^2 \\
    \quad &+ 2 p(A=1-a) |\phi^*(\hat x_{1-a}, 1-a) - \phi^*(x_{1-a}, 1-a) (\pf(x, a) - \phi^*(x, a)| + C \\
    \overset{(a)}{\leq}&
    p(A=1-a)^2 L^2 \varepsilon^2 + 2p(A=1-a) L \varepsilon |\pf(x, a) - \phi^*(x, a)| + C,
    \end{align*}
    where $C$ denotes the remaining term that only depends on $\pf$. 
    Here $(a)$ holds from the fact that the counterfactual estimation error is bounded by $\varepsilon$ and the assumption that $\phi^*$ is $L$-lipschitz. 

    Next, take the outer expectation,
    \begin{align*}
    \risk 
    &\leq 
     \E_{A}[p(A=1-a)^2] L^2 \varepsilon^2 + 2 L \varepsilon 
    \E_{X, A} [p(A=1-a)|\pf(x, a) - \phi^*(x, a)|] + \risk^*_\text{CF} \\
    &= 
    \sigma_A^2 L^2 \varepsilon^2 + 2L \varepsilon 
    \E_{X, A} [ p(A=1-a)^2 |\phi^*(x_{1-a}, 1-a) - \phi^*(x, a)|] + \risk^*_\text{CF}.
    \end{align*}
    Note that taking expectation of $C$ with respect to the joint distribution of $X, A$ is in fact the optimal risk $\risk^*_\text{CF}$. 
    Reorganization gives us
    \begin{align*}
    &
    \risk - \risk^*_\text{CF} \\
    \leq& 
     \sigma_A^2  L^2 \varepsilon^2 + 2L \varepsilon 
   \E_{X, A} [p(A=1-a)^2 |\phi^*(x_{1-a}, 1-a) - \phi^*(x, a)|] \\
    =&
    \sigma_A^2  L^2 \varepsilon^2 + 2L \varepsilon 
     \E_{X, A} [p(A=1-a)^2 | \E[Y \mid X=x_{1-a}, A=1-a] - \E[Y \mid X=x, A=a]|] \\
    =&
    \sigma_A^2  L^2 \varepsilon^2 + 2L \varepsilon 
     \E_{U, A} [p(A=1-a)^2|\E[Y \mid U=u, A=1-a] - \E[Y \mid U=u, A=a]|] \\
    =&
     \sigma_A^2  L^2 \varepsilon^2 + 2L \varepsilon 
     \E_{U} [p(A=1)p(A=0)^2|\E[Y \mid U=u, A=0] - \E[Y \mid U=u, A=1]|  \\
      & + p(A=0)p(A=1)^2|\E[Y \mid U=u, A=1] - \E[Y \mid U=u, A=0]| ]  \\
      =&    
    \sigma_A^2  L^2 \varepsilon^2 + 2 \sigma_A^2 L \varepsilon 
     \E_{U} [|\E[Y \mid U=u, A=1] - \E[Y \mid U=u, A=0]|]. \\\\
    \end{align*}

When $\ell$ is cross-entropy loss, the excess risk is
\begin{align*}
   & \risk - \risk^*_\CF \\
    =& \E_{X, A} \E_{Y \mid X=x, A=a} \left[ -Y \log\frac{\pp}{\pf} - (1-Y)\log\frac{1-\pp}{1-\pf} \right]
\end{align*}

Here we assume the logit (i.e., the inverse function of sigmoid) is $L$-lipschitz continuous in $x$, i.e., 
\begin{align*}
    \left| f^*(x,a) - f^*(\hat{x},a) \right | \leq L \| x - \hat x \|, \forall x, \hat x, a
\end{align*}

where $\phi^* \triangleq \sigma \circ f^* $ and $f^*(x,a) = \log \frac{\phi^*(x, a)}{1 - \phi^*(x, a)}$.
Now we check the excess risk. The first term can be upper bounded by 
\begin{align*}
    Y \log \frac{\phi^*_{\text{CF}}}{\phi_{\pcf}}
    &= 
    Y \log \frac{p(A=a) \phi^*(x, a) + p(A=1-a) \phi^*(x_{1-a}, 1-a) }{p(A=a) \phi^*(x, a) + p(A=1-a) \phi^*(\hat x_{1-a}, 1-a)} \\
    &\leq 
    Y (\phi^*_{\text{CF}}) ^{-1} \left( p(A=a) \phi^*(x, a) \log \frac{p(A=a) \phi^*(x, a)}{p(A=a) \phi^*(x, a)} \right.\\
    & \quad + \left. p(A=1-a) \phi^*(x_{1-a}, 1-a) \log \frac{p(A=1-a) \phi^*(x_{1-a}, 1-a)}{p(A=1-a) \phi^*(\hat x_{1-a}, 1-a)} \right) \\
    &=
    Y \frac{p(A=1-a) \phi^*(x_{1-a}, 1-a)}{\phi^*_{\text{CF}}} \log \frac{\phi^*(x_{1-a}, 1-a)}{\phi^*(\hat x_{1-a}, 1-a)} \\
&=
    Y C_1 \log \frac{\phi^*(x_{1-a}, 1-a)}{\phi^*(\hat x_{1-a}, 1-a)},
\end{align*}
where the inequality holds by applying log sum inequality. For the inequality, it is derived from
\begin{align*}
    \phi^*_{\text{CF}} \log \frac{ \phi^*_{\text{CF}}}{\phi_\pcf} \leq &p(A=a) \phi^*(x, a) \log \frac{p(A=a) \phi^*(x, a)}{p(A=a) \phi^*(x, a)} \\
    & \quad + p(A=1-a) \phi^*(x_{1-a}, 1-a) \log \frac{p(A=1-a) \phi^*(x_{1-a}, 1-a)}{p(A=1-a) \phi^*(\hat x_{1-a}, 1-a)} 
\end{align*}

And $C_1$ is defined as below
\begin{align*}
   C_1 & = \frac{p(A=1-a) \phi^*(x_{1-a}, 1-a)}{ P(A=a) \phi^*(x,a) + P(A=1-a)\phi^*(x_{1-a},1-a)} 
\end{align*}

Similarly,
\begin{align*}
    (1 - Y) \log \frac{1 - \phi^*_{\text{CF}}}{1 - \phi_{\pcf}}
    & =  (1-Y) \log \frac{p(A=a) - p(A=a) \phi^*(x, a) + p(A=1-a) - p(A=1-a)\phi^*(x_{1-a}, 1-a) }{p(A=a) - p(A=a) \phi^*(x, a) + p(A=1-a) -p(A=1-a) \phi^*(\hat x_{1-a}, 1-a)}\\
    &\leq 
    (1-Y) \frac{p(A=1-a) (1 - \phi^*(x_{1-a}, 1-a))}{1 - \phi^*_{\text{CF}}} \log \frac{1 - \phi^*(x_{1-a}, 1-a)}{1 - \phi^*(\hat x_{1-a}, 1-a)} \\
    &= 
    (1-Y) C_2 \log \frac{1 - \phi^*(x_{1-a}, 1-a)}{1 - \phi^*(\hat x_{1-a}, 1-a)}. 
\end{align*}
where 
\begin{align*}
    C_2 &= \frac{p(A=1-a) (1 - \phi^*(x_{1-a}, 1-a))}{1 - \phi^*_{\text{CF}}}\\
     & =\frac{p(A=1-a) (1 - \phi^*(x_{1-a}, 1-a))}{1 - P(A=a) \phi^*(x,a) - P(A=1-a)\phi^*(x_{1-a},1-a)}\\
     & =\frac{p(A=1-a) (1 - \phi^*(x_{1-a}, 1-a))}{p(A=a) - P(A=a) \phi^*(x,a)+ p(A=1-a) (1 - \phi^*(x_{1-a}, 1-a))}\\
\end{align*}

Put together
\begin{align*}
    &Y \log \frac{\phi^*_{\text{CF}}}{\phi_{\pcf}} + (1 - Y) \log \frac{1 - \phi^*_{\text{CF}}}{1 - \phi_{\pcf}}  \\
    \leq&
    Y C_1 \log \frac{\phi^*(x_{1-a}, 1-a)}{\phi^*(\hat x_{1-a}, 1-a)} + (1-Y) C_2 \log \frac{1 - \phi^*(x_{1-a}, 1-a)}{1 - \phi^*(\hat x_{1-a}, 1-a)} \\
    \leq&
    \left| Y C_1 \log \frac{\phi^*(x_{1-a}, 1-a)}{\phi^*(\hat x_{1-a}, 1-a)} + (1-Y) C_2 \log \frac{1 - \phi^*(x_{1-a}, 1-a)}{1 - \phi^*(\hat x_{1-a}, 1-a)}      \right| \\
    \overset{(a)}{\leq }& 
    \left| Y C_1 \log \frac{\phi^*(x_{1-a}, 1-a)}{\phi^*(\hat x_{1-a}, 1-a)} - (1-Y) C_2 \log \frac{1 - \phi^*(x_{1-a}, 1-a)}{1 - \phi^*(\hat x_{1-a}, 1-a)}      \right| \\
    \overset{(b)}{\leq }& 
    \max\left\{ Y C_1, (1-Y) C_2\right\} \left| \log \frac{\phi^*(x_{1-a}, 1-a)}{\phi^*(\hat x_{1-a}, 1-a)} - \log \frac{1 - \phi^*(x_{1-a}, 1-a)}{1 - \phi^*(\hat x_{1-a}, 1-a)}\right| \\
    \leq&
     \left| \log \frac{\phi^*(x_{1-a}, 1-a)}{\phi^*(\hat x_{1-a}, 1-a)} - \log \frac{1 - \phi^*(x_{1-a}, 1-a)}{1 - \phi^*(\hat x_{1-a}, 1-a)}\right| \\
     =& 
     \left| \log \frac{\phi^*(x_{1-a}, 1-a)}{1 - \phi^*(x_{1-a}, 1-a) } - \log \frac{\phi^*(\hat x_{1-a}, 1-a)}{1 - \phi^*(\hat x_{1-a}, 1-a)} \right|
     \leq L \| x_{1-a} - \hat x_{1-a} \|.
\end{align*}
Step $(a)$ holds from the observation that the two $\log$ terms must have different signs: unless $\phi^*(x_{1-a}, 1-a) = \phi^*(\hat x_{1-a}, 1-a) = 0.5$, otherwise it is impossible to have both 
\begin{align*}
    \phi^*(x_{1-a}, 1-a) &\geq \phi^*(\hat x_{1-a}, 1-a) \\
    1 - \phi^*(x_{1-a}, 1-a) &\geq 1 -  \phi^*(\hat x_{1-a}, 1-a),
\end{align*}
hold simultaneously. 
Step $(b)$ holds from the fact that the two terms now have the same sign so we can safely upper bound them, and this maximum is upper bounded by 1.  
Finally, the excess risk 
\begin{align*}
    \E_{X, A} \E_{Y \mid X=x, A=a} \left[ Y \log \frac{\phi^*_{cf}}{\phi_{pcf}} + (1 - Y) \log \frac{1 - \phi^*_{cf}}{1 - \phi_{pcf}} \right] \leq  \E_{X, A} \E_{Y \mid X=x, A=a} \left[ L \| x_{1-a} - \hat x_{1-a} \| \right] \leq L\varepsilon,
\end{align*}
so long as $\| x_{1-a} - \hat x_{1-a} \| \leq \varepsilon$. 
This completes the proof. 
\end{proof}

\section{Counterfactual Risk Minimization}\label{app-sec:crm}
\begin{theorem}\label{thm:crm}
  Given that $G$ is the ground truth counterfactual generating mechanism, CRM and PCF will have the same optimal solution under the constraint of perfect Counterfactual Fairness, i.e.,

$$\argmin_{\phi:\mathrm{TE}(\phi)=0}   \E_{X,A,Y} [\ell(\phi(X,A),Y)] =  \argmin_{\phi:\mathrm{TE}(\phi)=0}   \E_{X,A,Y} [\ell(\phi(X,A),Y) + \ell(\phi(G(X,A,1-A),A),Y)] $$
\end{theorem}

\begin{proof}

    \begin{align*}
         & \E_{X,A,Y} [ \ell(\phi(X,A),Y) + \ell(\phi(G(X,A,1-A),1-A),Y)]\\
         =& \E_{U,A,Y} [\ell(\phi(F^*_X(U,A),A),Y) + \ell(\phi(G(F^*_X(U,A),A,1-A),1-A),Y)]\\
         =& \E_{U,A,Y} [\ell(\phi(F^*_X(U,A),A),Y) + \ell(\phi(F^*_X(U,1-A),1-A),Y)]\\
         =& \E_{U} [
        p(A=a)\E_{Y|U=u,A=a}[\ell(\phi(F^*_X(u,a),a), Y)] + p(A=1-a)\E_{Y|U=u,A=1-a}[\ell(\phi(F^*_X(u,1-a),1-a), Y)]\\
        & +  p(A=a)\E_{Y|U=u,A=a}[\ell(\phi(F^*_X(u,1-a),1-a), Y)] + p(A=1-a)\E_{Y|U=u,A=1-a}[\ell(\phi(F^*_X(u,a),a), Y)]
         ]\\
                  =& \E_{X} [
        p(A=a)\E_{Y|X=x,A=a}[\ell(\phi(x,a), Y)] + p(A=1-a)\E_{Y|X=x_{1-a},A=1-a}[\ell(\phi(x_{1-a},1-a), Y)]\\
        & + p(A=a)\E_{Y|X=x,A=a}[\ell(\phi(x_{1-a},1-a), Y)] + p(A=1-a)\E_{Y|X=x_{1-a},A=1-a}[\ell(\phi(x,a), Y)]
         ]
\end{align*}
    Enforcing the constraint of CF, we define $\phi_0  \triangleq \phi(x,a) = \phi(x_{1-a},1-a)$, then the optimization problem of the inner expectation becomes
    \begin{align*}
    \argmin_{\phi_0} &  p(A=a)\E_{Y|X=x,A=a}[\ell(\phi_0, Y)] + p(A=1-a)\E_{Y|X=x_{1-a},A=1-a}[\ell(\phi_0, Y)] \\ 
    & +  p(A=a)\E_{Y|X=x,A=a}[\ell(\phi_0, Y)] + p(A=1-a)\E_{Y|X=x_{1-a},A=1-a}[\ell(\phi_0, Y)] \\
      = \argmin_{\phi_0} &  2 p(A=a)\E_{Y|X=x,A=a}[\ell(\phi_0, Y)] + 2 p(A=1-a)\E_{Y|X=x_{1-a},A=1-a}[\ell(\phi_0, Y)] 
\end{align*}
where the objective is just a scaled version of that in the proof of \Cref{thm:opt-predictor}. 
Hence, we get the same minimizer.
\end{proof}

\section{Experiment Details} \label{app-sec:exp-detail}

We included the codes to reproduce our results.
All GPU related experiments are run on RTX A5000.

\subsection{Dataset} \label{app-sec:dataset}
\paragraph{Synthetic Dataset}
In this section, we consider the two regression synthetic datasets and two classification tasks where all of our assumptions in \Cref{asm:main} are satisfied.
\begin{equation*}
\begin{aligned}[c]
    & \textit{Linear-Reg}\\
    A &\sim \textnormal{Bernoulli}(p_A), U \sim \mathcal{N}(0,1), \epsilon_Y \sim \mathcal{N}(0,1)\\
    X & =  w_A A + w_U U\\
    Y &  = w_X X + w'_U U + w_Y \epsilon_Y\\
\end{aligned}
\quad 
\begin{aligned}[c]
    & \textit{Cubic-Reg}\\
    A &\sim \textnormal{Bernoulli}(p_A), U \sim \mathcal{N}(0,1), \epsilon_Y \sim \mathcal{N}(0,1)\\
    X & =  w_A A + w_U U\\
    Y &  = w_X X^3 + w'_U U + w_Y \epsilon_Y\\
\end{aligned}
\end{equation*}
where in our experiments the parameters are chosen as $w_A=1, w_X=1, w_Y=1, w_U=1, w'_U=1$.

We also consider the following two classification tasks
\begin{equation*}
\begin{aligned}[c]
    & \textit{Linear-Cls}\\
A &\sim \textnormal{Bernoulli}(p_A), U \sim \mathcal{N}(0,1), \epsilon_Y \sim \mathcal{N}(0,1)\\
    X & =  w_A A + w_U U\\
    Y &  \sim \textnormal{Bernoulli}( \sigma(w_X X + w'_U U + w_Y \epsilon_Y ))\\
\end{aligned}
\quad 
\begin{aligned}[c]
    & \textit{Cubic-Cls}\\
A &\sim \textnormal{Bernoulli}(p_A), U \sim \mathcal{N}(0,1), \epsilon_Y \sim \mathcal{N}(0,1)\\
    X & =  w_A A + w_U U\\
    Y &  \sim \textnormal{Bernoulli}( \sigma(w_X X^3 + w'_U U + w_Y \epsilon_Y ))\\
\end{aligned}
\end{equation*}
where in our experiments the parameters are chosen as $w_A=2, w_X=1, w_Y=1, w_U=1, w'_U=1$.

\paragraph{Semi-synthetic Dataset}

We consider Law School Success \citep{wightman1998lsac}.
The sensitive attribute is gender and the target is first-year grade.
Other features contain race, LSAT and GPA.
However, since we need to evaluate TE of each method which requires access to ground truth, we use the simulated version of those datasets.
Following a similar setup in \citep{zuo2023counterfactually}, we train a generative model to get semi-synthetic datasets.
Specifically, we train a VAE with the following structure $u\sim Enc(x,a)$, $x\sim Dec_1(u,a)$, $y\sim Dec_2(u,x)$. 
The training objective includes a normal VAE objective to reconstruct $x$ via $Enc$ and $Dec_1$, and a supervised objective to generate $y$ via $Dec_2$.
After training, we first sample a prior $u \sim \mathcal{N}(0,I)$ and $a\sim Bernoulli(p)$ (where $p$ is acquired based on empirical frequency in real data), then we get the semi-synthetic $x,y$ using $Dec_1$ and $Dec_2$.
We want to emphasize that counterfactuals, regardless of train or test set, are hidden from downstream models and used for evaluation only.
This way, we get access to the ground truth $u^*$ and can generate ground truth counterfactuals without any error.
In our investigation, exogenous noise, factual data and counterfactual data are all actually the simulated version of original datasets.
However they do follow a fixed data generating mechanism that is close to the real data.

\subsection{Analytic Solution on Synthetic Datasets}\label{app-sec:ana-solution}

We know the analytic solution of Bayes optimal predictor in our synthetic data experiments.
More specifically, for \textit{Linear-Reg}, we have 
$$
\phi^*(x,a) = \mathbb{E}[Y|X=x,A=a] = w_X x + \frac{w'_U}{w_U} (x-w_A a)
$$
For $\textit{Cubic-Reg}$, we have 
$$
\phi^*(x,a) = \mathbb{E}[Y|X=x,A=a] = w_X x^3 + \frac{w'_U}{w_U} (x-w_A a)
$$
For \textit{Linear-Cls}, we have 
$$
\phi^*(x,a) = \mathbb{E}[Y|X=x,A=a] = \sigma(w_X x + \frac{w'_U}{w_U} (x-w_A a))
$$
For $\textit{Cubic-Cls}$, we have 
$$
\phi^*(x,a) = \mathbb{E}[Y|X=x,A=a] = \sigma(w_X x^3 + \frac{w'_U}{w_U} (x-w_A a))
$$

\subsection{Prediction Models}

In our synthetic experiments, we mainly use KNN based predictors.
We use the default parameters in \texttt{scikit-learn}.
All MLP methods uses a structure with hidden layer $(20,20)$ and Tanh activation.

In semi-synthetic experiments, we use MLP methods uses a structure with hidden layer $(5,5)$ and Tanh activation as this is closer to the ground truth SCM.

\section{Additional Results} \label{app-sec:exp-result}

\subsection{Additional results on synthetic datasets}
In \Cref{fig:gt_mlp}, we test how how all algorithms perform when using ground truth counterfactuals and $U$ on additional type of predictors.
We observe that PCF achieves lower error than CFU and CFR, which is similar to what we observe in \Cref{fig:gt_knn}.
This further validates our theory regarding optimality of PCF.

In \Cref{fig:est_b0.001_knn}, following the investigation in \Cref{fig:est_knn}, we test with adding gaussian noise with different mean.
We observe that when it is a fixed bias, CFU and CFR achieves better fairness than PCF.
Though PCF still achieves best predictive performance.
Furthermore, as we increase variance of the noise, PCF outperform these two methods in terms of both fairness and ML performance.
In \Cref{fig:estana_b0.001_knn}, similar to \Cref{fig:estana_knn}, we observe PCF-Analytic significantly improves over PCF. 
Notably, it is not affected by bias as PCF.

\subsection{Additional results on semi-synthetic datasets}
In \Cref{fig:law_allte}, we included the expanded version of \Cref{fig:law} with $\te_0$ and $\te_1$.
We observe that they show a very similar trend.
In \Cref{fig:law-pcf-erm-crm}, we directly compare PCF (with ERM) and PCF-CRM. 
The results validate the necessity of CRM as a plugin estimator $\phi$ in the case of limited causal knowledge.

\begin{figure}[!ht]
    \centering
    \begin{subfigure}[t]{0.22\linewidth}
        \centering
        \includegraphics[width=\linewidth]{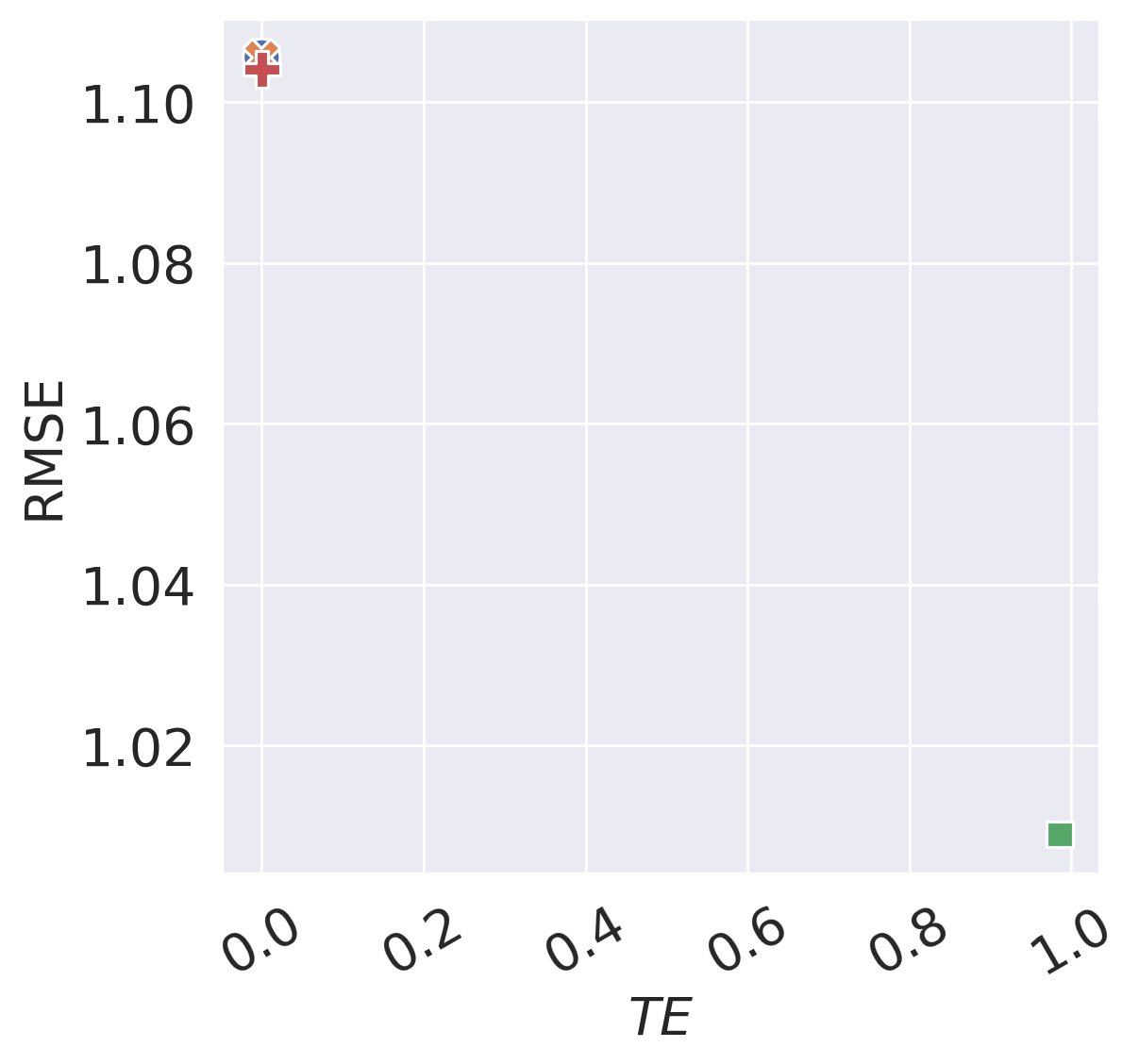}
        \caption{\textit{Linear-Reg}}
        \label{fig:gt_linear_knn_cf_effect}
    \end{subfigure}
    \begin{subfigure}[t]{0.27\linewidth}
        \centering
        \includegraphics[width=\linewidth]{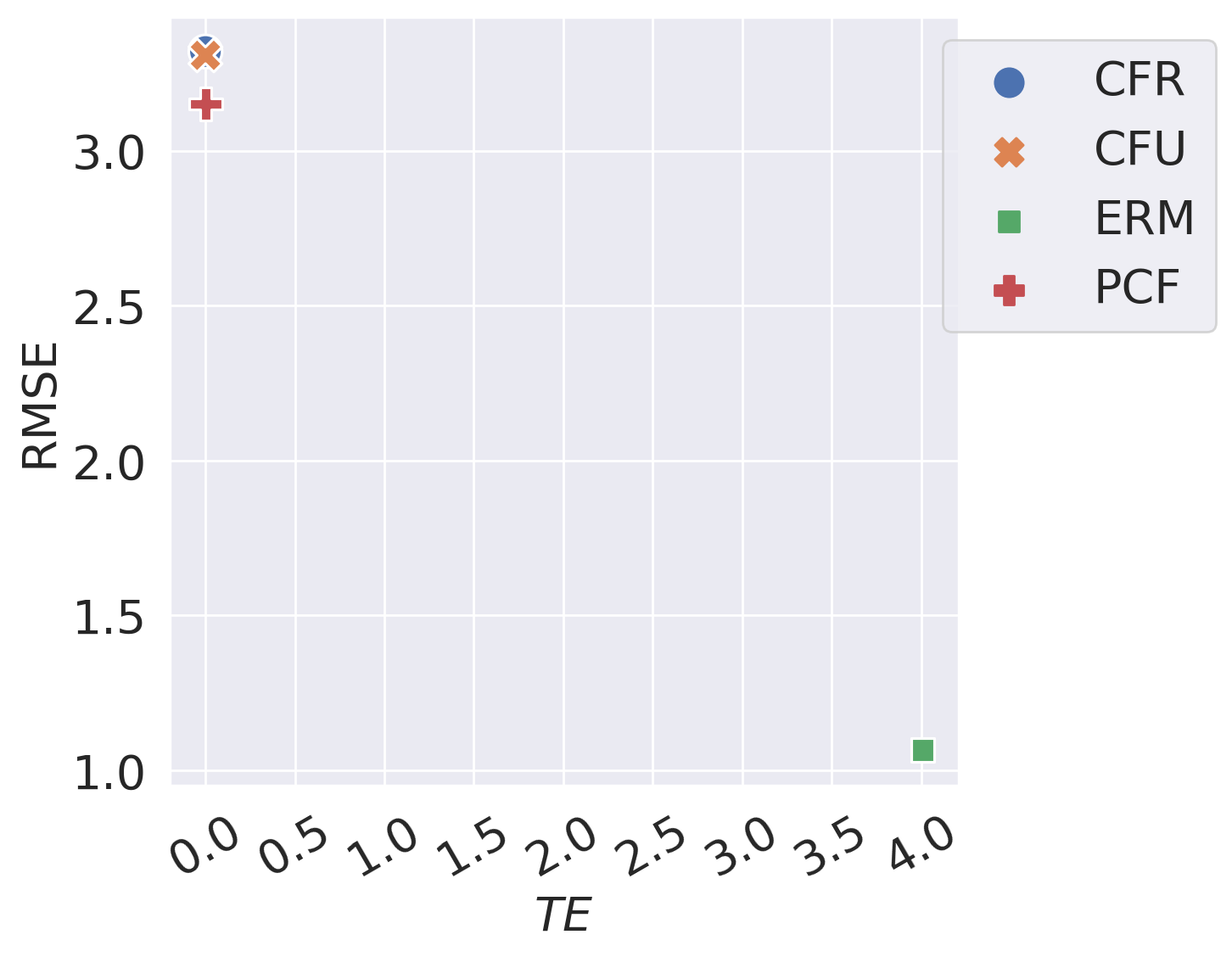}
        \caption{\textit{Cubic-Reg}}
        \label{fig:gt_cubic_knn_cf_effect}
    \end{subfigure}
    \begin{subfigure}[t]{0.22\linewidth}
        \centering
        \includegraphics[width=\linewidth]{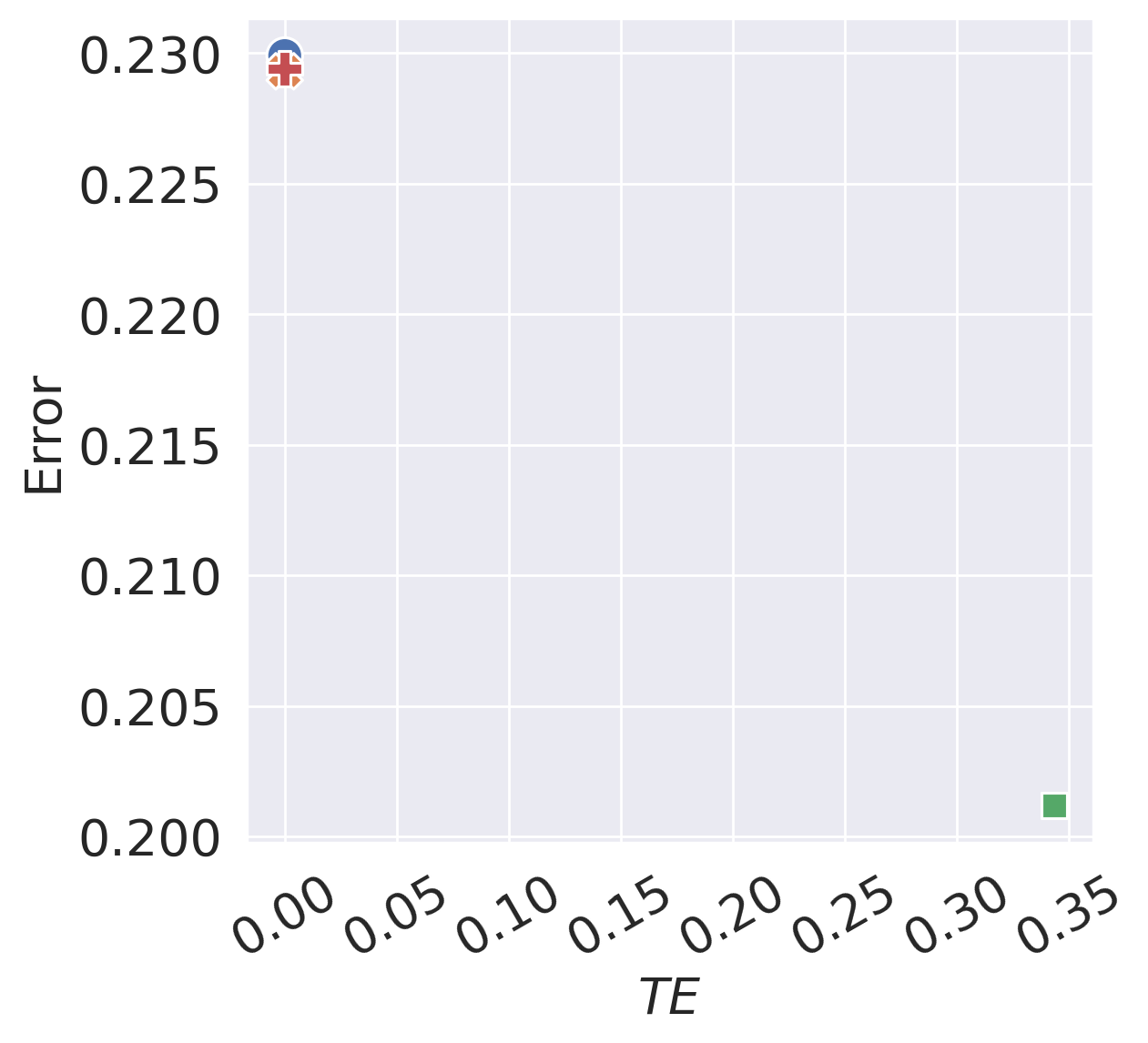}
        \caption{\textit{Linear-Cls}}
        \label{fig:gt_linear_knn_cf_effect}
    \end{subfigure}
    \begin{subfigure}[t]{0.27\linewidth}
        \centering
        \includegraphics[width=\linewidth]{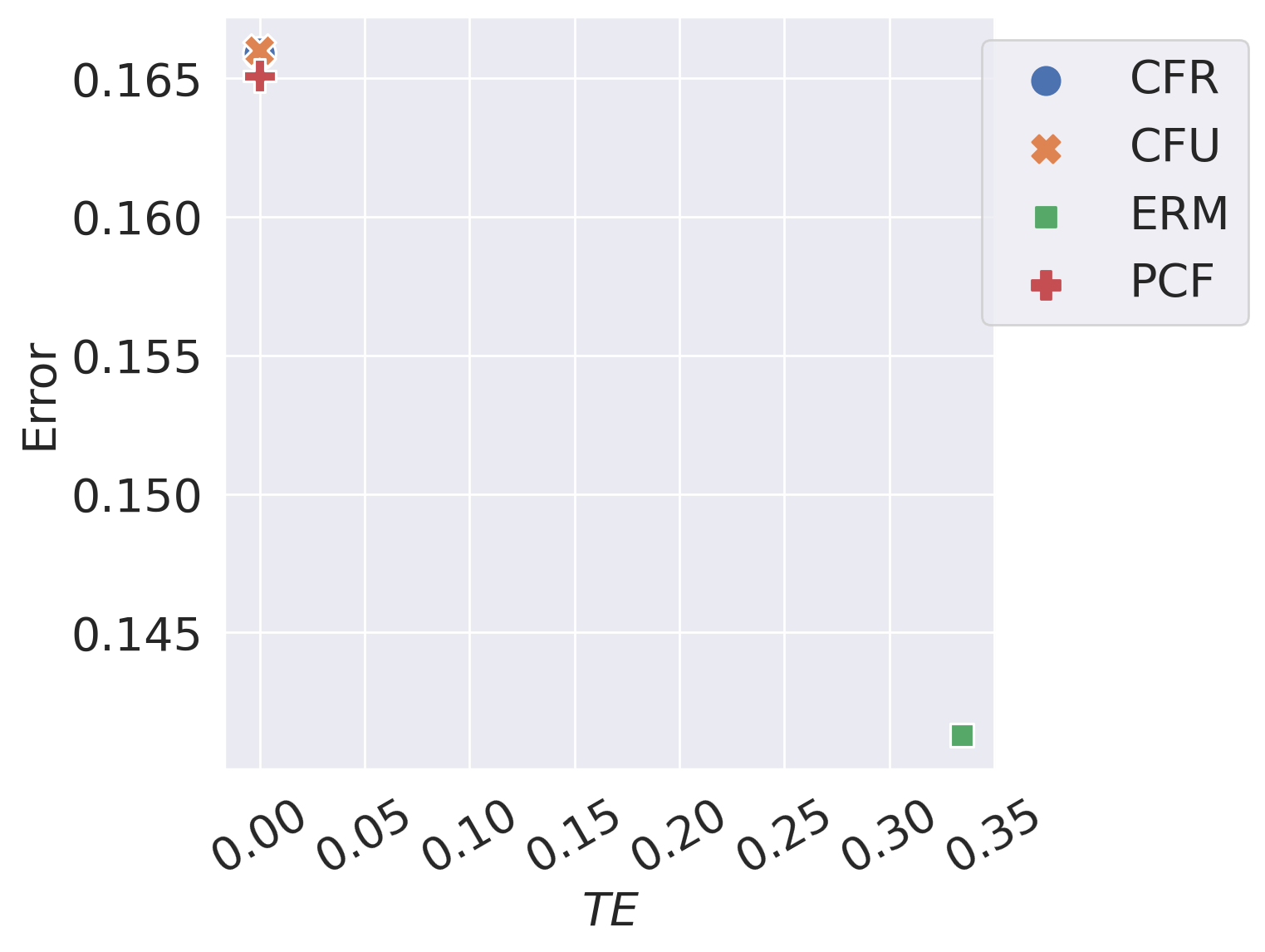}
        \caption{\textit{Cubic-Cls}}
        \label{fig:gt_cubic_knn_cf_effect}
    \end{subfigure}
    \caption{Results on synthetic datasets given ground truth counterfactuals with MLP predictors.}
    \label{fig:gt_mlp}
\end{figure}

\begin{figure}[!ht]
    \centering
    \begin{subfigure}[t]{0.22\linewidth}
        \centering
        \includegraphics[width=\linewidth]{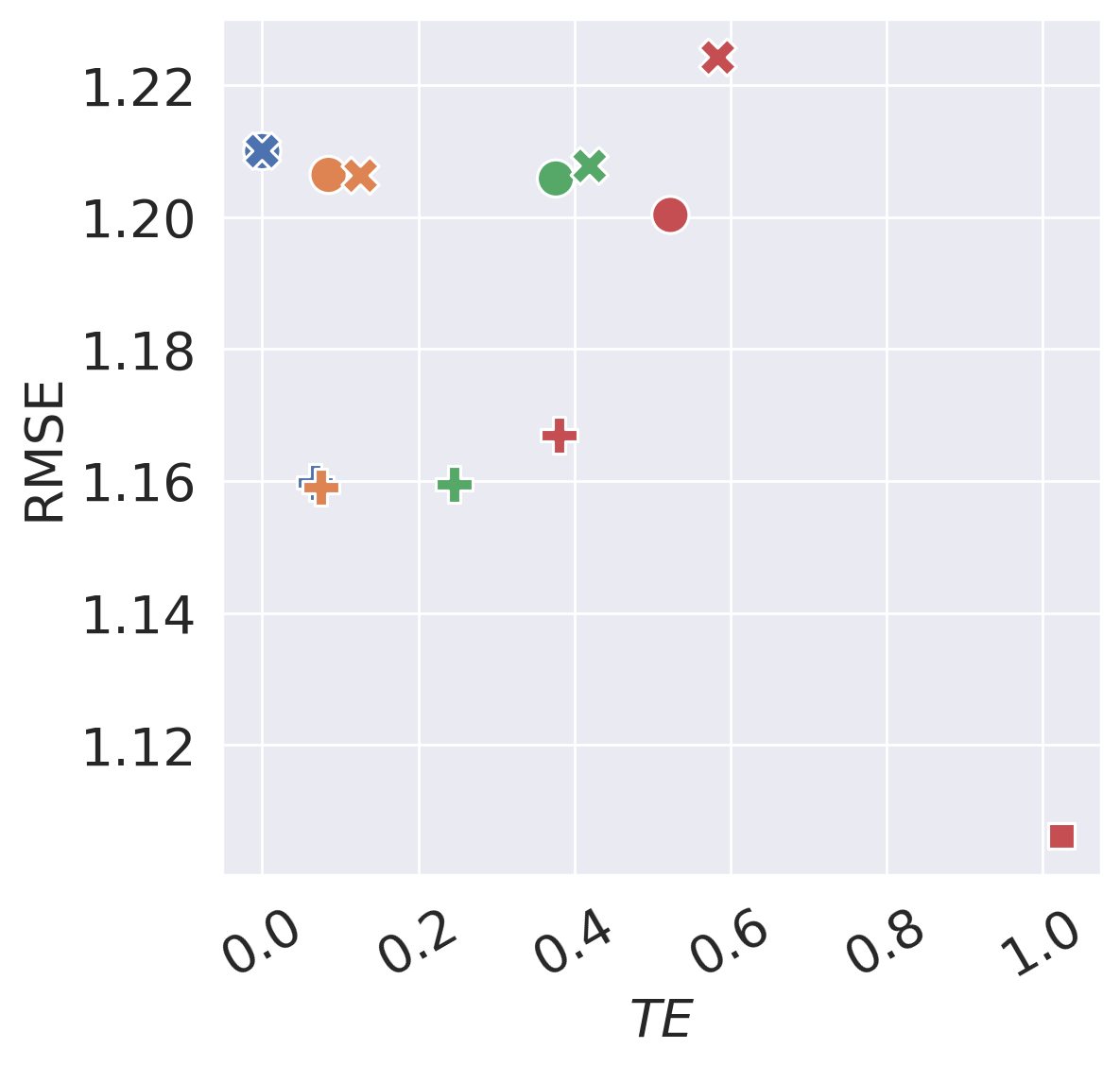}
        \caption{\textit{Linear-Reg}}
        \label{fig:est_linear_knn_cf_effect}
    \end{subfigure}
    \begin{subfigure}[t]{0.22\linewidth}
        \centering
        \includegraphics[width=\linewidth]{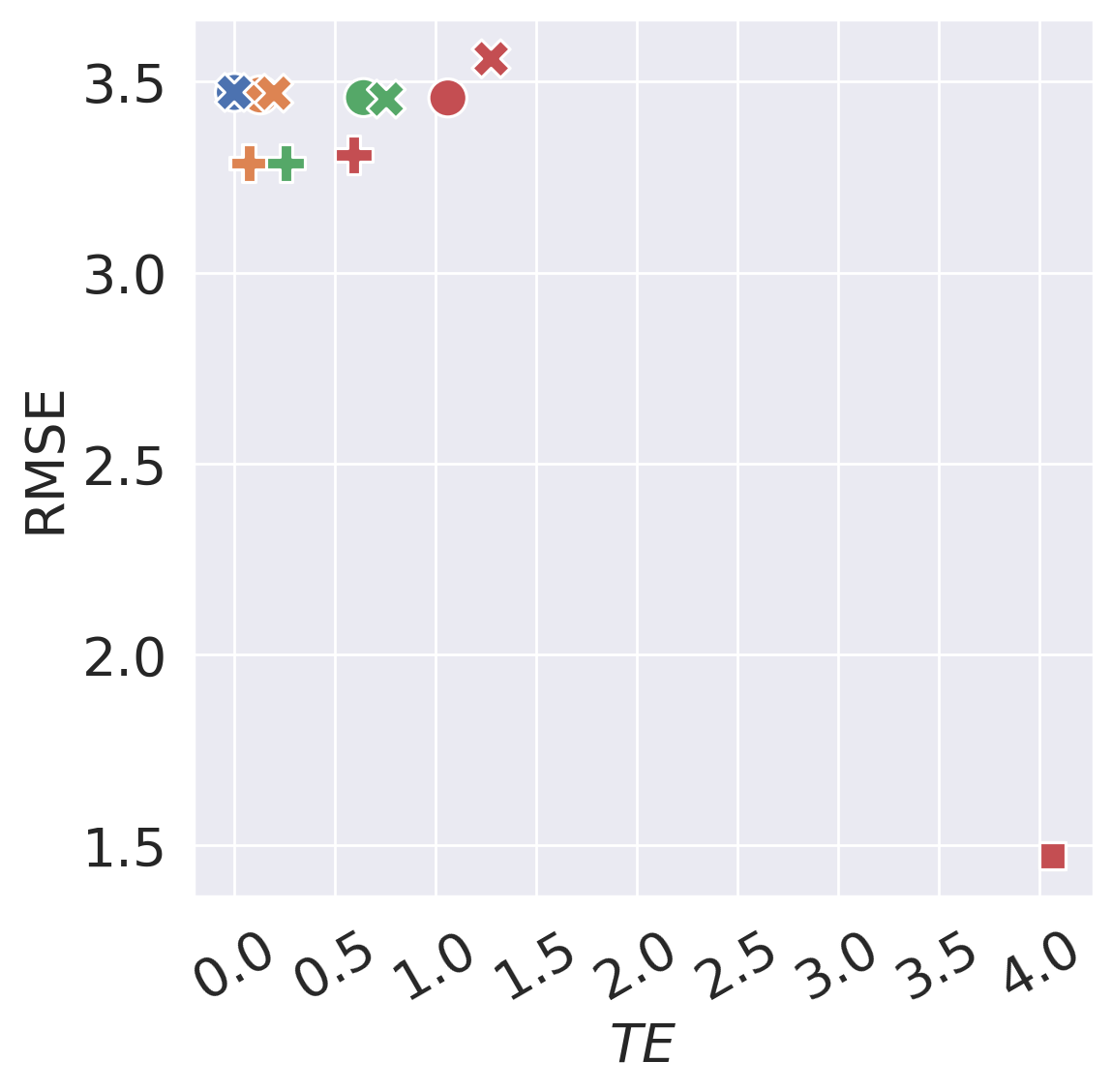}
        \caption{\textit{Cubic-Reg}}
        \label{fig:est_cubic_knn_cf_effect}
    \end{subfigure}
    \begin{subfigure}[t]{0.22\linewidth}
        \centering
        \includegraphics[width=\linewidth]{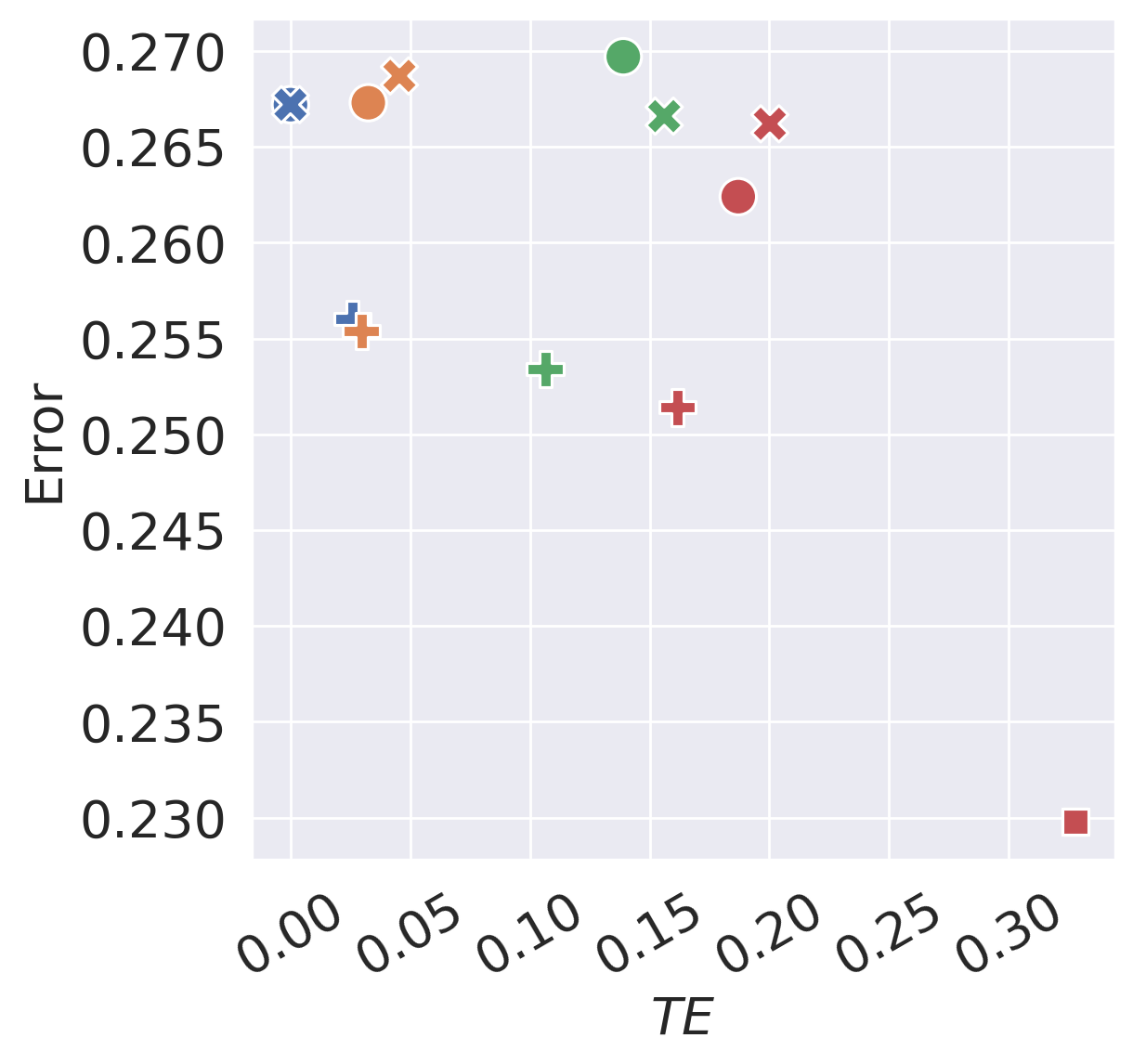}
        \caption{\textit{Linear-Cls}}
        \label{fig:est_linear_knn_cf_effect}
    \end{subfigure}
    \begin{subfigure}[t]{0.28\linewidth}
        \centering
        \includegraphics[width=\linewidth]{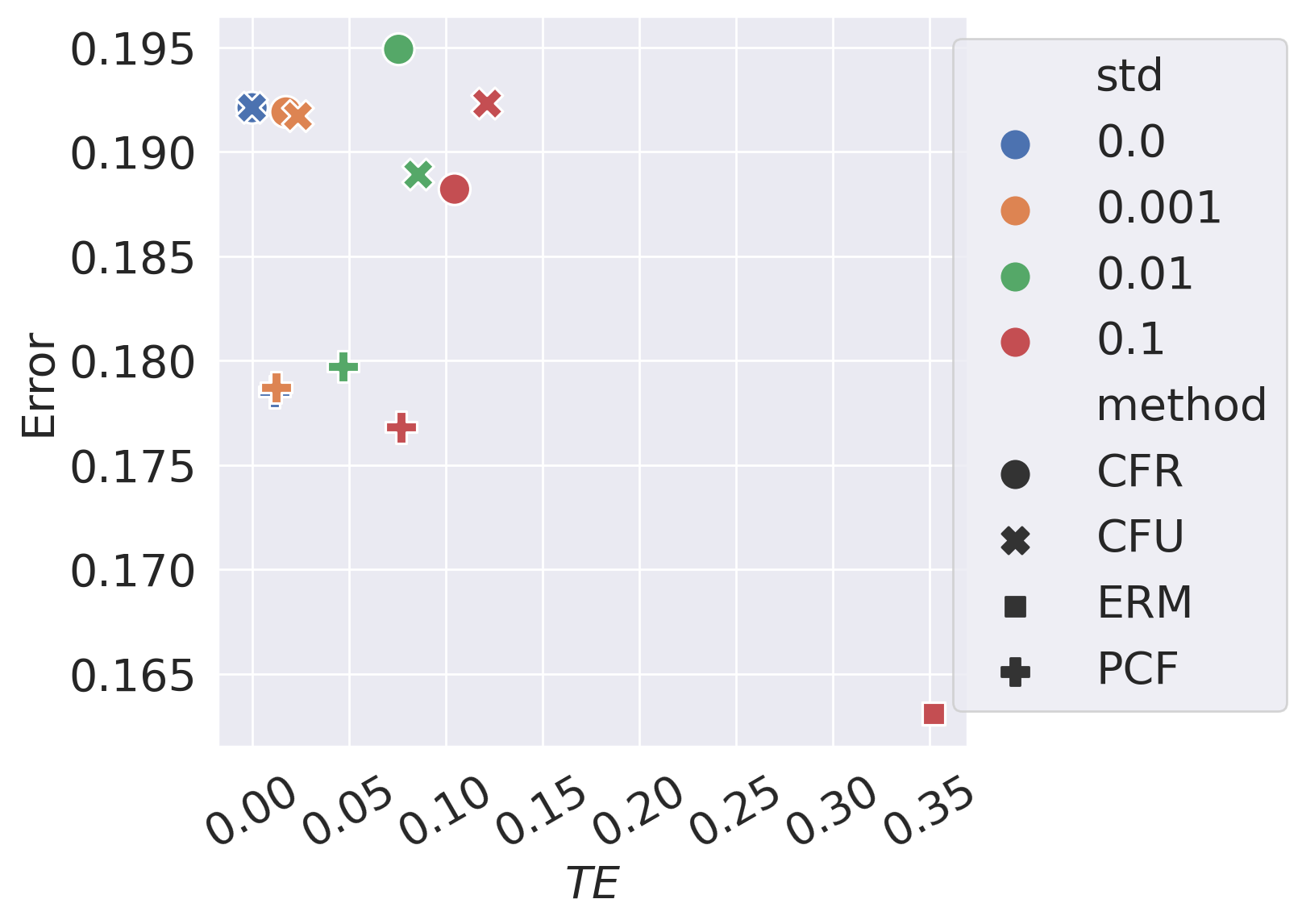}
        \caption{\textit{Cubic-Cls}}
        \label{fig:est_cubic_knn_cf_effect}
    \end{subfigure}
    \caption{Results on synthetic datasets under counterfactual estimation error with KNN predictors. 
    Different color represents different $\alpha$ indicating the standard deviation of the error ($\epsilon \sim \mathcal{N}(0.001,\alpha)$) while shape represents different algorithms.}
    \label{fig:est_b0.001_knn}
\end{figure}

\begin{figure}[!ht]
    \centering
    \begin{subfigure}[t]{0.22\linewidth}
        \centering
        \includegraphics[width=\linewidth]{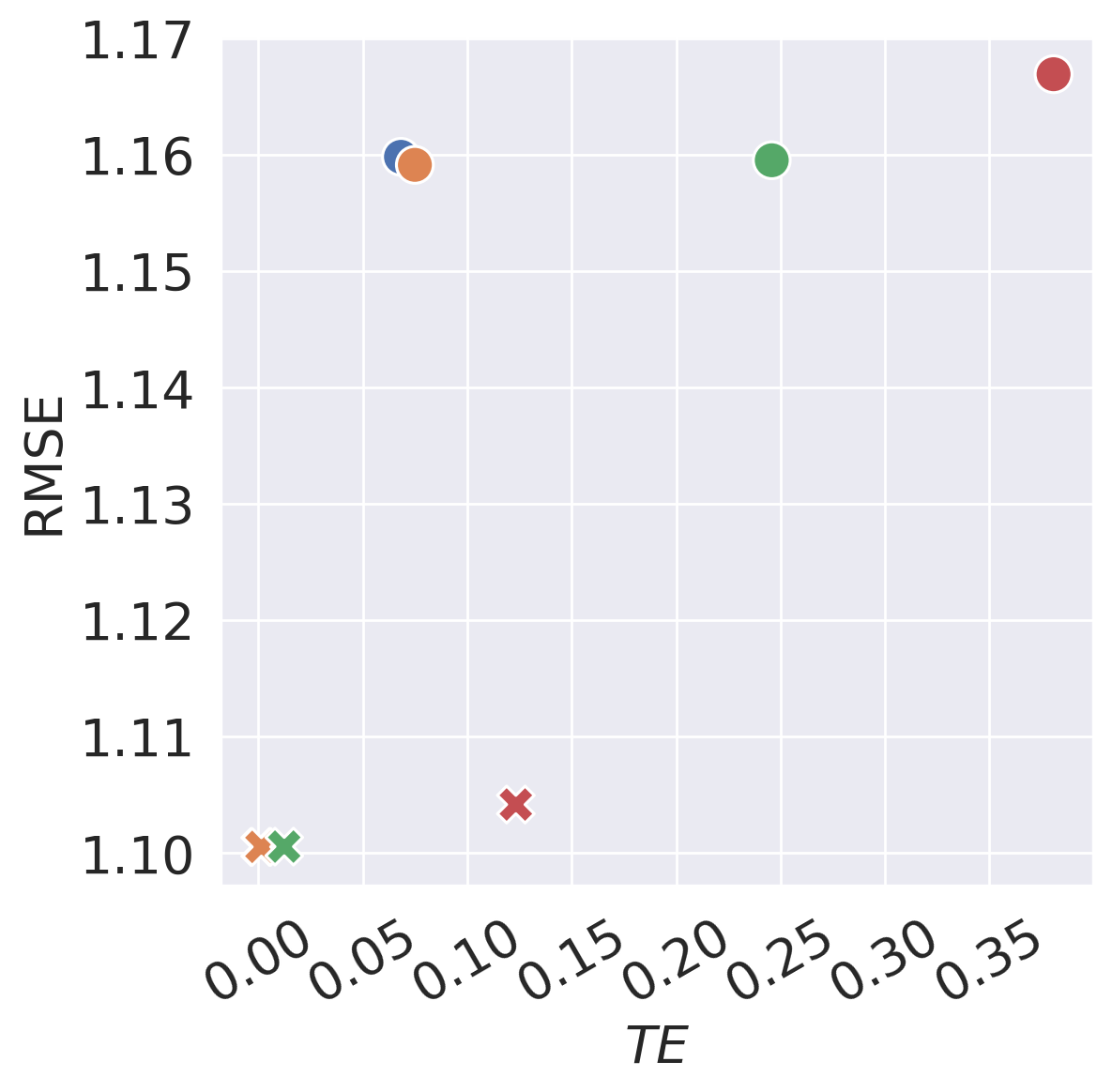}
        \caption{\textit{Linear-Reg}}
        \label{fig:estana_linear_knn_cf_effect}
    \end{subfigure}
    \begin{subfigure}[t]{0.22\linewidth}
        \centering
        \includegraphics[width=\linewidth]{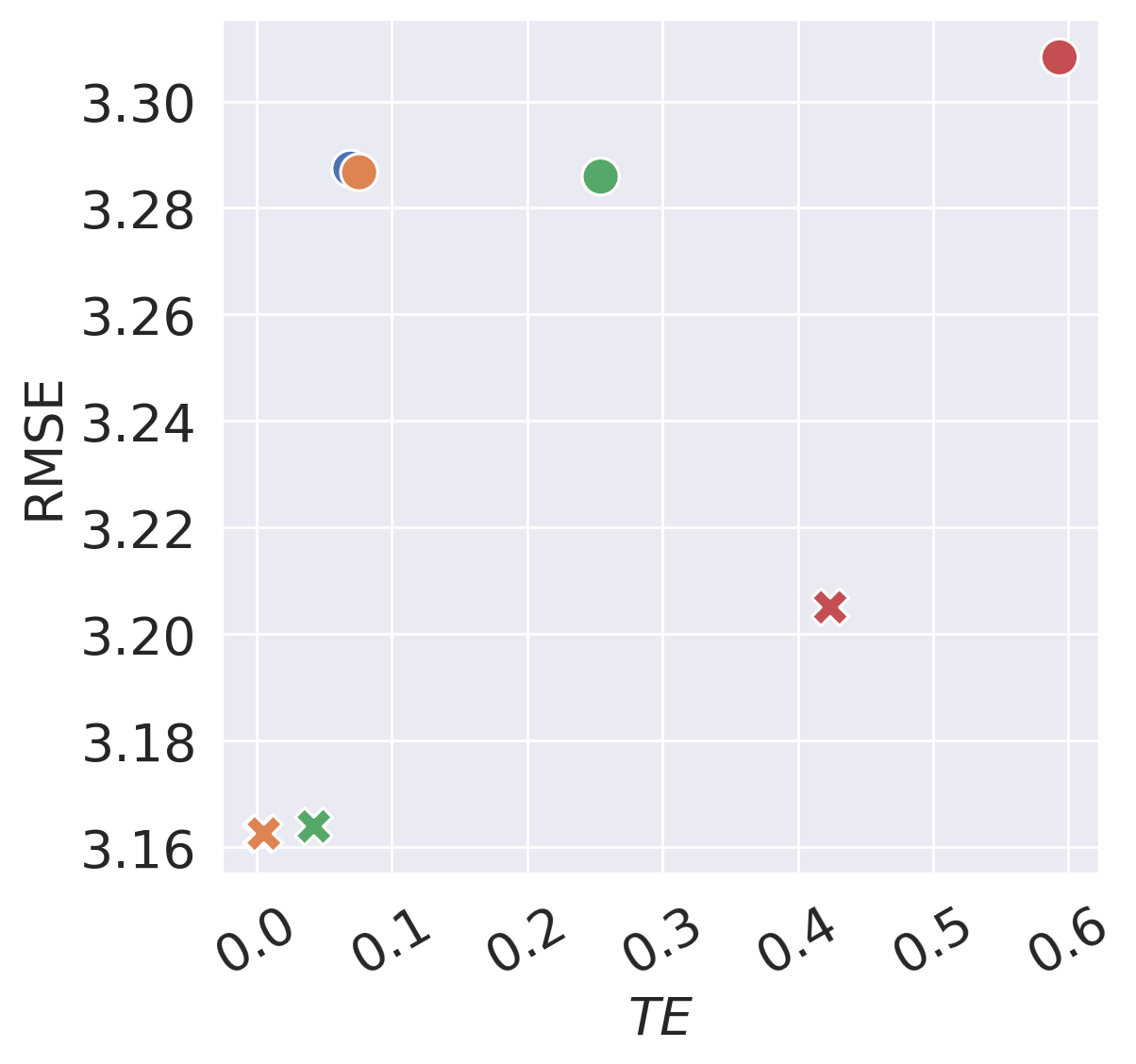}
        \caption{\textit{Cubic-Reg}}
        \label{fig:estana_cubic_knn_cf_effect}
    \end{subfigure}
    \begin{subfigure}[t]{0.22\linewidth}
        \centering
        \includegraphics[width=\linewidth]{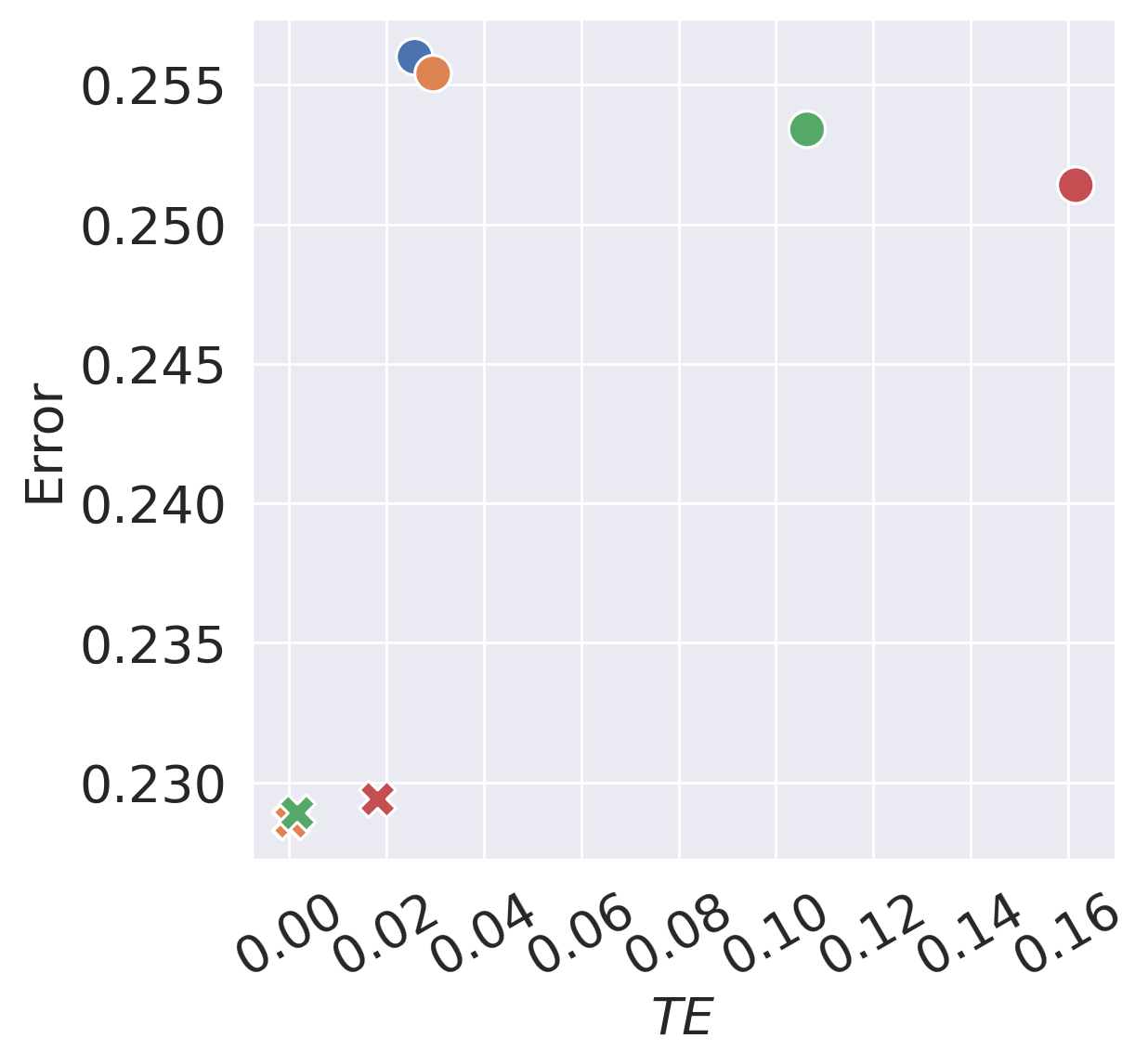}
        \caption{\textit{Linear-Cls}}
        \label{fig:estana_linear_knn_cf_effect}
    \end{subfigure}
    \begin{subfigure}[t]{0.28\linewidth}
        \centering
        \includegraphics[width=\linewidth]{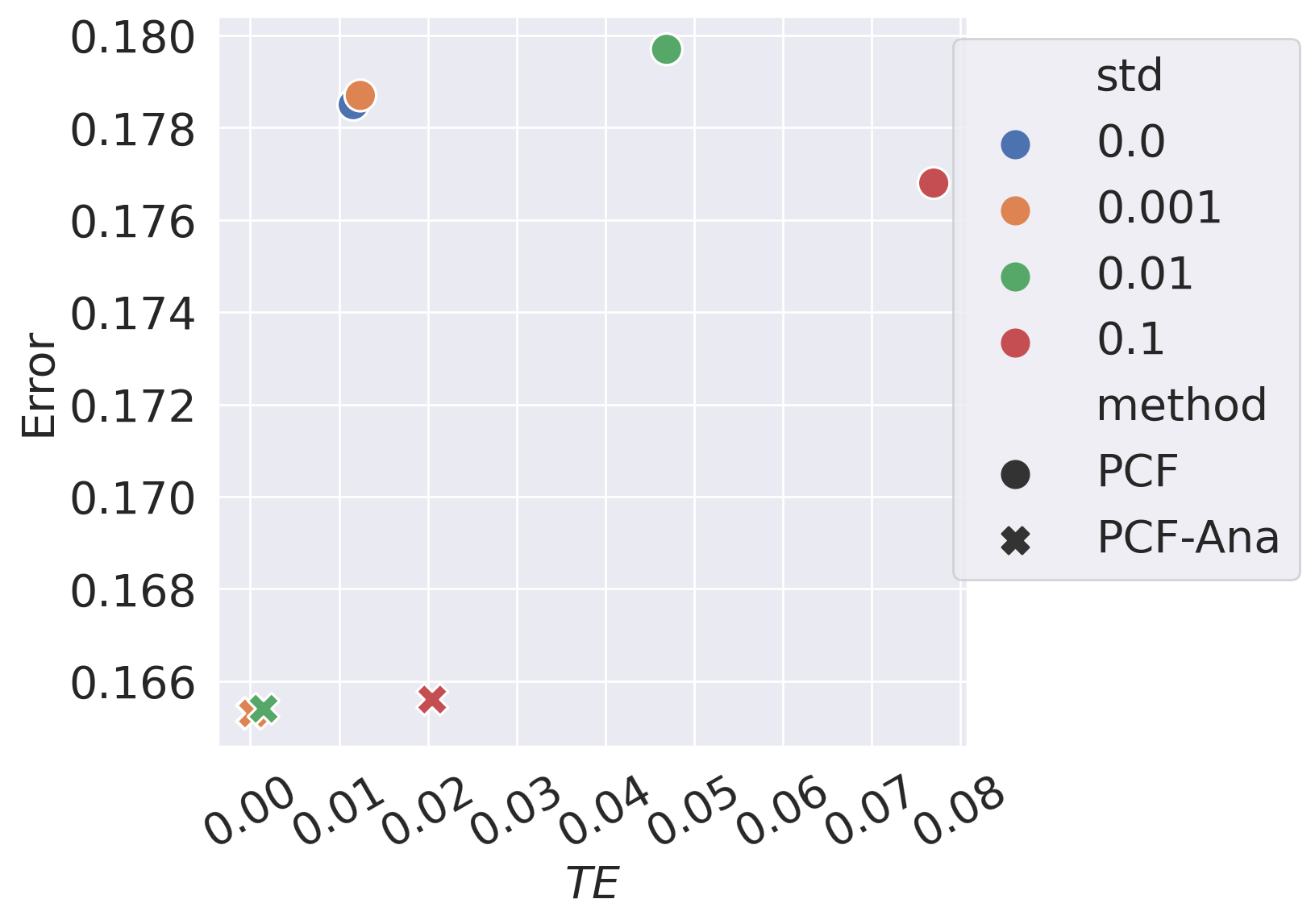}
        \caption{\textit{Cubic-Cls}}
        \label{fig:estana_cubic_knn_cf_effect}
    \end{subfigure}
    \caption{Results on synthetic datasets comparing PCF and PCF-Analytic with KNN predictors.
    Different color represents different $\alpha$ indicating the standard deviation of the error ($\epsilon \sim \mathcal{N}(0.001,\alpha)$) while shape represents different algorithms.}
    \label{fig:estana_b0.001_knn}
\end{figure}

\begin{figure}
    \centering
    \begin{subfigure}[t]{0.5\linewidth}
        \centering
        \includegraphics[width=\linewidth]{figures/law_official/estaug_mlp_cf_effect.png}
        \caption{$\TE$}
        \end{subfigure}
    \begin{subfigure}[t]{0.5\linewidth}
        \centering
        \includegraphics[width=\linewidth]{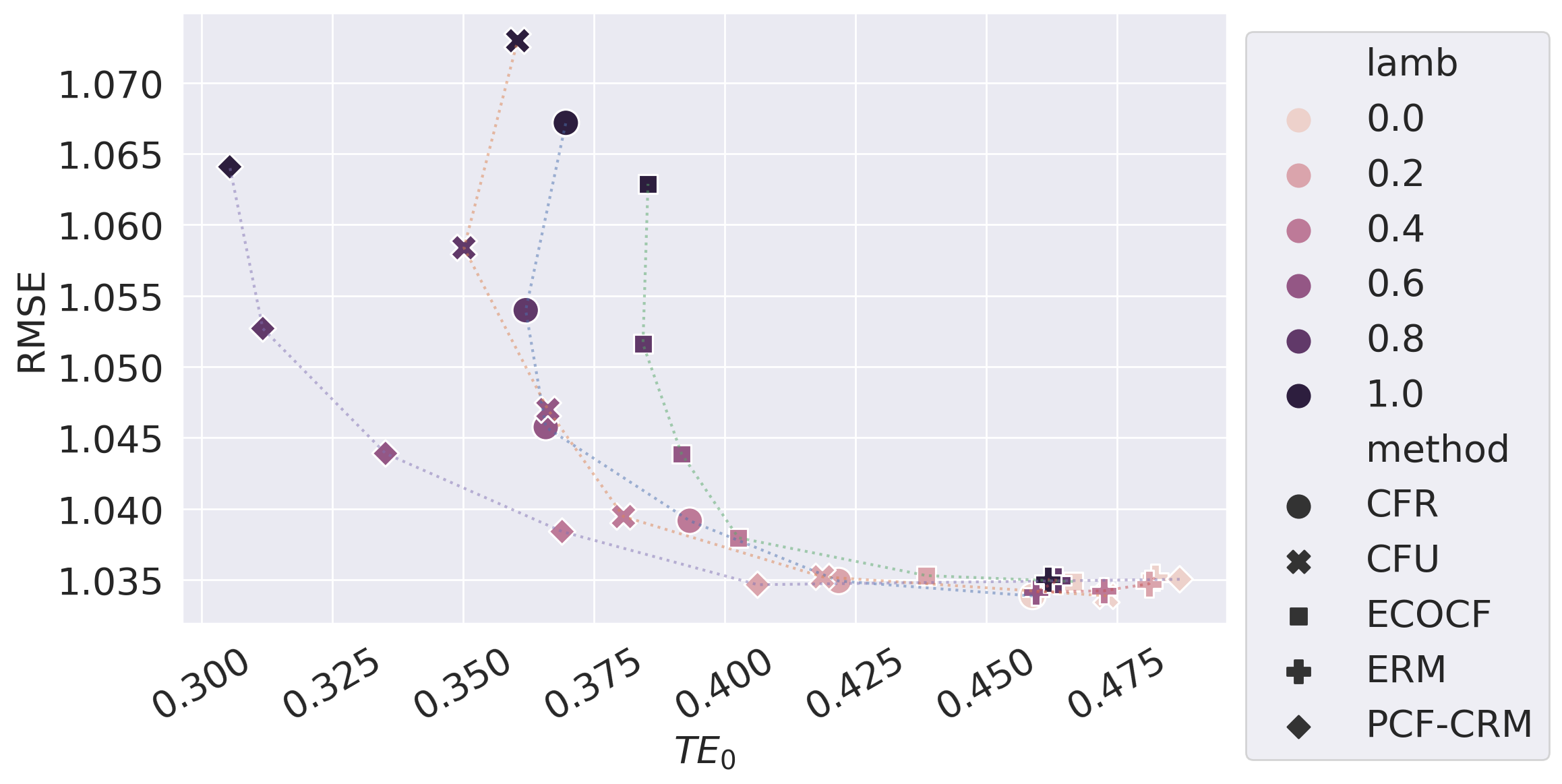}
        \caption{$\TE_0$}
        \end{subfigure}
    \begin{subfigure}[t]{0.5\linewidth}
        \centering
        \includegraphics[width=\linewidth]{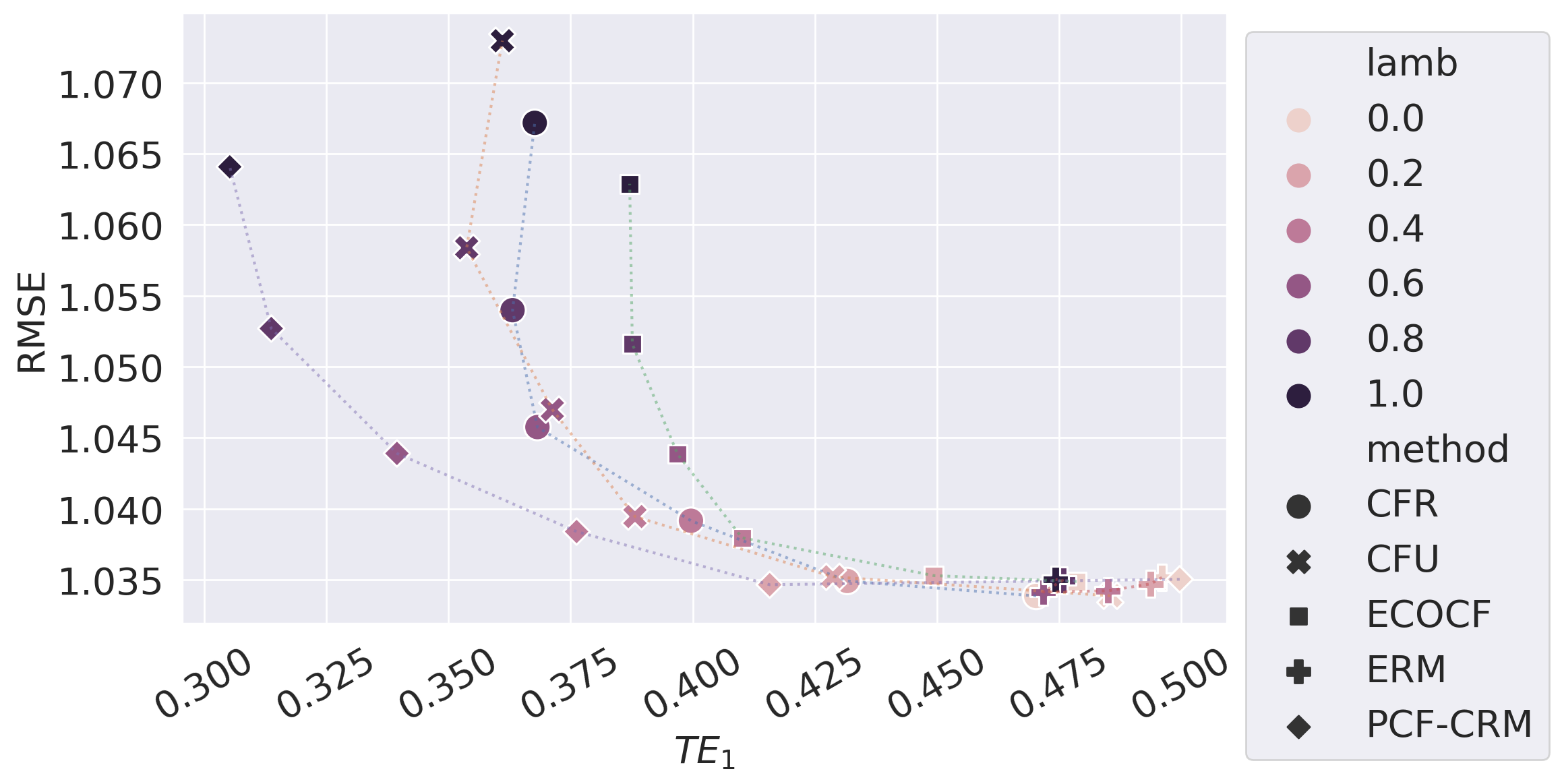}
        \caption{$\TE_1$}
        \end{subfigure}
    \caption{Results on Sim-Law with estimated counterfatuals.
    The predictor is a MLP regressor.
    We also test the convex combination of each algorithm and ERM.
    For example, PCFAug with $\lambda$ means $\hat{y} = \lambda \hat{y}_{\textnormal{PCFAug}} + (1-\lambda) \hat{y}_{\textnormal{ERM}}$.
    This suggests that PCFAug can achieve lower Error given the same TE and lower TE given the same Error.}
\label{fig:law_allte}
\end{figure}

\begin{figure}
    \centering
\includegraphics[width=0.5\linewidth]{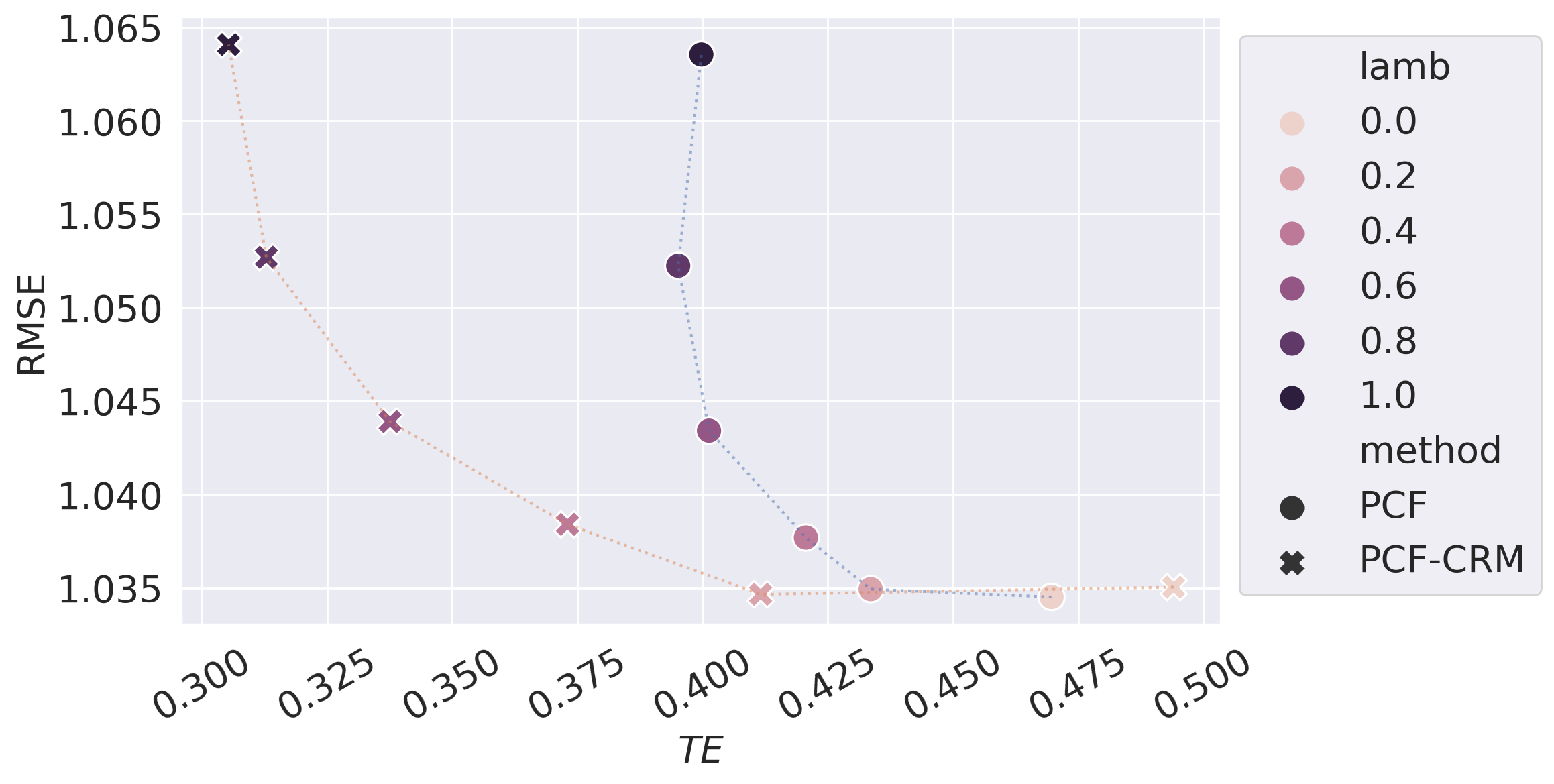}
\caption{Comparison between PCF and PCF-CRM on Sim-Law with estimated counterfatuals.
    The predictor is a MLP regressor.}
\label{fig:law-pcf-erm-crm}
\end{figure}

\begin{figure}
    \centering
        \includegraphics[width=0.75\linewidth]{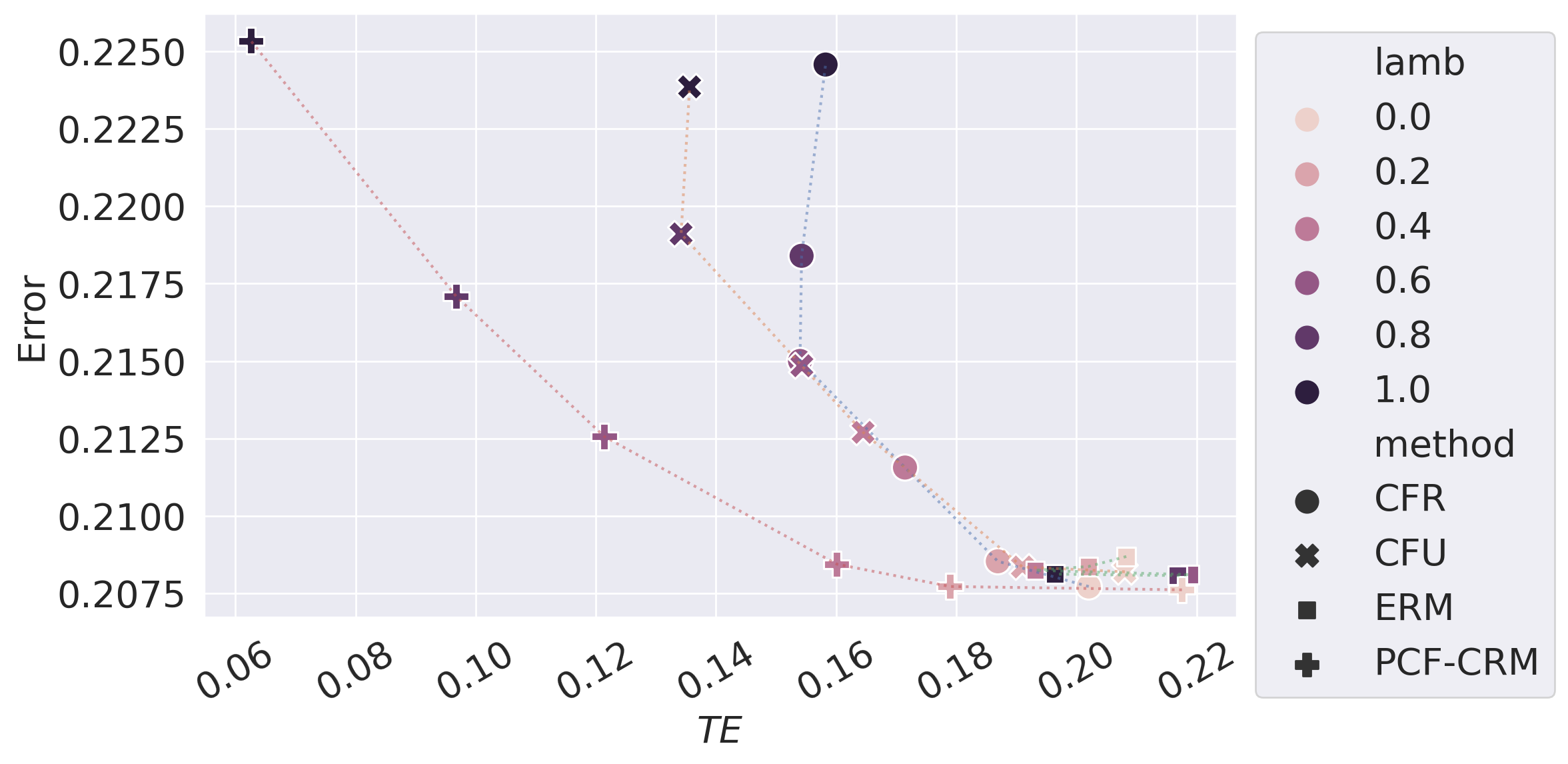}
    \caption{Results on Sim-Adult with estimated counterfatuals.
    The predictor is a MLP classifier.}
\label{fig:adult}
\end{figure}
 
\subsection{Additional experiments}

To evaluate the performance of our method across a broader range of data generating mechanisms, counterfactual estimation models, and datasets, we conducted experiments using the Disentangled Causal Effect Variational Autoencoder (DCEVAE) on the Adult dataset \citep{asuncion2007uci}.
Similarly, we trained one DCEVAE to simulate the Sim-Adult dataset and another DCEVAE to estimate counterfactuals. 
Here we followed the data preprocessing, model setup, and hyperparameter choices specified in \citet{zuo2023counterfactually}.
A key difference between the setup used here and our setup on the Law School dataset is that, in this case, the encoder of the ground truth DCEVAE takes $y$ as an input. 
This modification introduces an inconsistency between the ground truth CGM and the estimated CGM.
In \Cref{fig:adult}, we observe a trend similar to that in \Cref{fig:law}, which indicates the effectiveness of PCF-CRM.

\clearpage

\section*{NeurIPS Paper Checklist}

\begin{enumerate}

\item {\bf Claims}
    \item[] Question: Do the main claims made in the abstract and introduction accurately reflect the paper's contributions and scope?
    \item[] Answer: \answerYes{} \item[] Justification: The abstract and the introduction reflect the contributions and scope of the work presented in this paper.

\item {\bf Limitations}
    \item[] Question: Does the paper discuss the limitations of the work performed by the authors?
    \item[] Answer: \answerYes{} \item[] Justification: It is discussed in both \Cref{sec:fair-alg} and \Cref{sec:conclusion}.

\item {\bf Theory Assumptions and Proofs}
    \item[] Question: For each theoretical result, does the paper provide the full set of assumptions and a complete (and correct) proof?
    \item[] Answer: \answerYes{} \item[] Justification: See \Cref{sec:prelim} and \Cref{sec:fair-alg}.

    \item {\bf Experimental Result Reproducibility}
    \item[] Question: Does the paper fully disclose all the information needed to reproduce the main experimental results of the paper to the extent that it affects the main claims and/or conclusions of the paper (regardless of whether the code and data are provided or not)?
    \item[] Answer: \answerYes{} \item[] Justification: See \Cref{sec:exp} and \Cref{app-sec:exp-detail}.

\item {\bf Open access to data and code}
    \item[] Question: Does the paper provide open access to the data and code, with sufficient instructions to faithfully reproduce the main experimental results, as described in supplemental material?
    \item[] Answer: \answerYes{} \item[] Justification: The code is released.

\item {\bf Experimental Setting/Details}
    \item[] Question: Does the paper specify all the training and test details (e.g., data splits, hyperparameters, how they were chosen, type of optimizer, etc.) necessary to understand the results?
    \item[] Answer: \answerYes{} \item[] Justification: See \Cref{sec:exp} and \Cref{app-sec:exp-detail}.

\item {\bf Experiment Statistical Significance}
    \item[] Question: Does the paper report error bars suitably and correctly defined or other appropriate information about the statistical significance of the experiments?
    \item[] Answer: \answerNo{}{} \item[] Justification: It would be too messy to include error bars in the current plot.

\item {\bf Experiments Compute Resources}
    \item[] Question: For each experiment, does the paper provide sufficient information on the computer resources (type of compute workers, memory, time of execution) needed to reproduce the experiments?
    \item[] Answer: \answerYes{} \item[] Justification: See \Cref{app-sec:exp-detail}.
    
\item {\bf Code Of Ethics}
    \item[] Question: Does the research conducted in the paper conform, in every respect, with the NeurIPS Code of Ethics \url{https://neurips.cc/public/EthicsGuidelines}?
    \item[] Answer: \answerYes{} \item[] Justification: To the best of our knowledge we do not violate any portion of the NeurIPS Code of Ethics.

\item {\bf Broader Impacts}
    \item[] Question: Does the paper discuss both potential positive societal impacts and negative societal impacts of the work performed?
    \item[] Answer: \answerYes{} \item[] Justification: This paper aims to enhance Counterfactual Fairness in practical scenarios, but it's important to note that our method's validation is currently limited to semi-simulated datasets with known counterfactuals. Therefore, further auditing and verification are essential before applying this method in real-world fairness scenarios.
    
\item {\bf Safeguards}
    \item[] Question: Does the paper describe safeguards that have been put in place for responsible release of data or models that have a high risk for misuse (e.g., pretrained language models, image generators, or scraped datasets)?
    \item[] Answer: \answerNA{} \item[] Justification: We do not believe our work has a risk of misuse

\item {\bf Licenses for existing assets}
    \item[] Question: Are the creators or original owners of assets (e.g., code, data, models), used in the paper, properly credited and are the license and terms of use explicitly mentioned and properly respected?
    \item[] Answer: \answerYes{} \item[] Justification: The datasets used are cited.

\item {\bf New Assets}
    \item[] Question: Are new assets introduced in the paper well documented and is the documentation provided alongside the assets?
    \item[] Answer: \answerNA{}{} \item[] Justification: No new assets are released in the paper.

\item {\bf Crowdsourcing and Research with Human Subjects}
    \item[] Question: For crowdsourcing experiments and research with human subjects, does the paper include the full text of instructions given to participants and screenshots, if applicable, as well as details about compensation (if any)? 
    \item[] Answer: \answerNA{}{} \item[] Justification: The paper does not involve crowdsourcing.

\item {\bf Institutional Review Board (IRB) Approvals or Equivalent for Research with Human Subjects}
    \item[] Question: Does the paper describe potential risks incurred by study participants, whether such risks were disclosed to the subjects, and whether Institutional Review Board (IRB) approvals (or an equivalent approval/review based on the requirements of your country or institution) were obtained?
    \item[] Answer: \answerNA{} \item[] Justification: The paper does not involve crowdsourcing nor research with human subjects.

\end{enumerate}

\end{document}